\documentclass[11pt]{article}

\usepackage[margin=1in]{geometry}
\usepackage{amsmath,amssymb,amsthm}
\usepackage{bm}
\usepackage{algorithm}
\usepackage{algpseudocode}
\usepackage{float}
\usepackage{booktabs}
\usepackage{enumitem}
\usepackage{graphicx}
\usepackage[round,authoryear]{natbib}
\usepackage{microtype}
\usepackage{lastpage}
\usepackage[hypertexnames=false]{hyperref}

\hypersetup{
  colorlinks=true,
  linkcolor=blue,
  citecolor=blue,
  urlcolor=blue,
  pdftitle={Logit-Coordinate Generative Models for Mixed Continuous-Categorical Tabular Data},
  pdfauthor={Yuefei Shen and Xiaotong Shen},
  pdfkeywords={mixed tabular data, logit coordinates, flow matching, diffusion models, information geometry}
}

\graphicspath{{./}}

\newtheorem{theorem}{Theorem}[section]
\newtheorem{lemma}[theorem]{Lemma}
\newtheorem{proposition}[theorem]{Proposition}
\newtheorem{corollary}[theorem]{Corollary}
\theoremstyle{definition}
\newtheorem{definition}[theorem]{Definition}
\newtheorem{assumption}[theorem]{Assumption}
\theoremstyle{remark}
\newtheorem{remark}[theorem]{Remark}

\newenvironment{keywords}
  {\par\smallskip\noindent\textbf{Keywords:}\ }
  {\par\medskip}

\title{Logit--Coordinate Generative Models for Mixed Continuous--Categorical Tabular Data}

\author{
  \begin{tabular}{cc}
    Yuefei Shen & Xiaotong Shen\\
    \texttt{yms5084@gmail.com} & \texttt{xshen@umn.edu}
  \end{tabular}\\[0.5em]
  \normalsize School of Statistics, University of Minnesota\\
  \normalsize Minneapolis, MN 55455, USA
}
\date{}

\begin{document}

\maketitle

\begin{abstract}

Mixed continuous--categorical data pose a representation problem for continuous generative models. Flow Matching and Gaussian diffusion operate in Euclidean spaces, whereas categorical laws lie on probability simplices and may be highly imbalanced. We study a logit-coordinate framework that   encodes categorical variables as smoothed natural parameters  and combines them with transformed numerical variables. This yields common formulations of Logit Flow Matching and Logit Diffusion. We introduce a mixed-distribution discrepancy separating categorical marginal error from conditional continuous Wasserstein error, and derive stability bounds and imbalance-aware nonparametric rates linking vector-field or drift error to decoded mixed-distribution error. Controlled simulations show that scaled-logit coordinates improve or match one-hot coordinates, especially under severe rare-cell imbalance. Across four real-data benchmarks and ten splits per dataset, Logit FM improves the primary distributional metrics on three datasets and is comparable on Churn2; Block-Conditional Logit FM consistently improves the flat model; and Logit Diffusion generally improves over or matches One-Hot Diffusion.

\end{abstract}


\section{Introduction}

Mixed tabular data are neither purely Euclidean nor purely discrete. A typical record contains numerical variables, categorical attributes, and often a response label whose distribution is strongly imbalanced. If
\[
X=(X^{(c)},X^{(d)}),
\]
with $X^{(c)}\in\mathbb R^p$ and $X^{(d)}$ taking values in a finite set, a useful generative model must approximate both the categorical marginal law and the continuous law conditional on each category. These two requirements become coupled under imbalance: the categories that are hardest to preserve are also the categories with the fewest observations for learning conditional continuous structure.

Transport-based generative models offer an attractive route for this problem. Flow Matching (FM) \citep{lipman2022flow} learns a time-dependent velocity field whose ordinary differential equation transports a simple base distribution to the data distribution. Diffusion models use stochastic noising and denoising processes, but their probability-flow representations also fit into a vector-field view. For continuous data, this viewpoint has led to stability and convergence analyses that relate vector-field error to distributional error. For mixed data, the same argument cannot be applied without first choosing how categorical variables enter the transport space.

The representation choice is a modeling decision with statistical consequences. One-hot encodings place all categories at equal Euclidean distances and then ask a continuous vector field to interpolate through the ambient space. This construction is convenient, but it does not encode the changing local curvature of the probability simplex. Natural-parameter, or logit, coordinates give a different embedding: after smoothing categorical observations into the interior of the simplex, additive log-ratio coordinates map them to an unconstrained Euclidean space while preserving the exponential-family coordinate structure.

This paper studies logit generative models for mixed continuous--categorical tabular data. The central idea is representation-level: embed categorical variables in smoothed natural-parameter coordinates, concatenate them with transformed numerical variables, and then apply a continuous generative model in the resulting Euclidean coordinate system. Flow Matching and diffusion are treated as two instantiations of this principle. The geometry is used in this coordinate sense: the proposed models use Euclidean dynamics in natural-parameter coordinates, with no Fisher--Rao training objective. The contribution is the combination of this classical log-ratio representation with a mixed-distribution discrepancy, stability analysis, imbalance-aware rates, and within-family empirical comparisons for mixed tabular generation.

\subsection{Generative Modeling for Mixed Tabular Data}

Recent research has extended generative models to tabular data in various forms. 
Early approaches adapted generative adversarial networks (GANs) to heterogeneous tabular settings \citep{xu2019ctgan, zhao2021ctabgan}. These models incorporate conditional sampling mechanisms. They also introduce feature-wise normalization to handle mixed variable types and class imbalance.
Variational autoencoder (VAE) frameworks take a similar approach by introducing differentiable relaxations, for example, the Gumbel--Softmax, to model categorical variables \citep{jang2017gumbel, maddison2017concrete}.

Diffusion and transport-based models have also been extended to tabular data.
Approaches like TabDDPM \citep{kotelnikov2023tabddpm} split continuous and categorical features into separate diffusion processes. 
Related approaches instead treat one-hot encodings as continuous vectors perturbed by Gaussian noise.
In parallel, methods such as Argmax Flows \citep{hoogeboom2021argmax} introduce continuous latent representations that are discretized later to recover categorical variables.

Despite their empirical success, these approaches mostly rely on Euclidean embeddings of
categorical variables, typically through one-hot vectors or learned continuous representations.
The geometric consequences of this choice are rarely discussed explicitly, particularly when class proportions are highly imbalanced. 

\subsection{Related Work and Positioning}

Several recent papers address categorical or mixed-type generation with flow or diffusion models. \citet{hoogeboom2021argmax} develop Argmax Flows and multinomial diffusion for categorical data. TabDDPM \citep{kotelnikov2023tabddpm} combines Gaussian diffusion for numerical variables with multinomial diffusion for categorical variables. More recent flow-matching approaches treat heterogeneous tabular data through variational or cascaded constructions, including exponential-family variational Flow Matching \citep{guzman2025efvfm} and cascaded Flow Matching for heterogeneous tables \citep{mueller2026cascaded}. A separate line studies categorical generation directly on, or through transformations of, the simplex. Statistical Flow Matching \citep{cheng2024sfm} uses Fisher--Rao geometry on statistical manifolds, Categorical Flow Maps \citep{roos2026categorical} use simplex-constrained endpoint parameterizations, and simplex-to-Euclidean methods use log-ratio-type bijections to make categorical flow matching compatible with standard Euclidean models \citep{williams2025simplex}.

This paper is closest in spirit to the simplex-to-Euclidean and exponential-family viewpoints, but it asks a different statistical question. We focus on mixed continuous--categorical tabular distributions under imbalance, where fidelity requires preserving both categorical mass and numerical distributions conditional on categorical-label cells. The proposed logit representation is therefore studied as a representation principle for continuous generative modeling of mixed data. Flow Matching provides one instantiation through logit-space velocity fields, while diffusion provides another through Gaussian noising and denoising in the same transformed coordinates.
This framing separates a coordinate-representation effect that can otherwise be bundled with architecture, conditioning, and benchmark design in mixed-tabular generation.

\subsection{Continuous Generative Dynamics and Geometric Structure}

The primary obstacle in applying continuous generative models to mixed tabular data is geometric. Flow Matching is defined through deterministic velocity fields in Euclidean space, while Gaussian diffusion models perturb and denoise variables in Euclidean coordinates. Even when diffusion is viewed through its probability-flow ODE representation, the generative dynamics operate through continuous vector fields. Thus, both Flow Matching and diffusion require categorical variables to be embedded into a continuous coordinate system before they can be modeled jointly with numerical variables.

Categorical distributions reside on the probability simplex
\[
\Delta_K = \{ p \in \mathbb{R}^K: p_k \ge 0,\ \sum_{k=1}^K p_k = 1 \},
\]
which carries a non-Euclidean information geometry described by the   Fisher--Rao  metric \citep{amari2016information}. 
A common practice embeds categorical variables using one-hot encodings and treats them as continuous during training. Although it is computationally convenient, this approach implicitly treats the simplex as a flat Euclidean subset of  $\mathbb{R}^K$. Consequently, this Euclidean relaxation potentially distorts statistical structure and degrades rare-category fidelity.

Categorical distributions form an exponential family, and the natural parameters correspond to logits. 
The logit transformation maps the interior of the simplex to an unconstrained Euclidean space. In these coordinates, exponential-family structure is represented through natural parameters, which gives a statistically interpretable Euclidean embedding for transport-based generative modeling. 
This motivates a logit-coordinate representation of categorical variables for continuous generative dynamics.

\subsection{Contributions}

Building on this perspective, we develop a logit-coordinate framework for generative modeling of mixed continuous--categorical tabular data. The contribution is fivefold.

\medskip
\noindent\textit{(1) Logit generative modeling for mixed data.}
We embed categorical variables in smoothed natural-parameter coordinates and concatenate them with transformed numerical variables, producing a continuous representation of mixed tabular data. This representation can be used with different continuous generative dynamics, including Flow Matching and Gaussian diffusion.

\medskip
\noindent\textit{(2) Geometry-aware categorical representation.}
The proposed representation replaces one-hot Euclidean coordinates with logit coordinates associated with the categorical exponential family. This gives a simple unconstrained Euclidean training space while preserving the natural-parameter structure of categorical distributions.

\medskip
\noindent\textit{(3) A mixed-distribution fidelity criterion.}
We introduce a discrepancy for mixed distributions that decomposes categorical marginal error and conditional continuous Wasserstein error. This discrepancy reflects the intrinsic mixture structure of mixed tabular data and directly evaluates whether a generator preserves both category proportions and within-category numerical distributions.

\medskip
\noindent\textit{(4) Unified stability analysis for logit generative models.}
We establish stability bounds showing how errors in learned continuous dynamics propagate to mixed-distribution error after decoding from logit coordinates. For Flow Matching, this gives a velocity-field stability bound. For diffusion models, the same mechanism applies through the probability-flow ODE representation, where score-estimation error induces drift-field error.

\medskip
\noindent\textit{(5) Imbalance-aware statistical rates and empirical validation.}
By reducing the learning problem to nonparametric estimation of time-dependent vector fields, we derive convergence rates that make the effective-sample-size loss for rare categories explicit. Controlled simulations compare logit and one-hot representations under rare-cell imbalance, while real-data experiments over ten data splits per dataset evaluate the representation effect and the benefit of additional block-conditional structure.

Together, these results support a representation-level message: for mixed tabular generation under imbalance, the categorical coordinate system can materially affect distributional fidelity across continuous generative modeling frameworks. The empirical comparisons remain primarily within family, while the two TabDDPM variants serve as external references rather than a basis for direct cross-family ranking.

\section{Mathematical Background}

In this section, we formalize the statistical and geometric foundations for mixed discrete--continuous generative modeling.

\noindent\textit{Notation.}
Unless stated otherwise, a superscript $\star$ marks a population or
data-generating (``true'') quantity and a hat marks its learned, estimated, or
generated counterpart; for example, $(P^\star,\widehat P)$ and
$(\pi_k^\star,\widehat\pi_k)$.

\subsection{Mixed Discrete--Continuous Probability Spaces}

Let $(\Omega, \mathcal{F}, \mathbb{P})$ be a probability space.
We consider random variables
\[
X = (X^{(c)}, X^{(d)}),
\]
where
\[
X^{(c)} \in \mathbb{R}^p, 
\quad
X^{(d)} \in \mathcal{D},
\]
and $\mathcal{D}$ is a finite discrete space with total cardinality $K$; for multiple categorical variables, $K$ denotes the number of joint categorical cells.

A generic joint density admits the following factorization:
\[
 p(x^{(c)}, x^{(d)})
 =
 p(x^{(d)}) \, p(x^{(c)} \mid x^{(d)}).
\]
This representation identifies the space of mixed distributions with 
\[
\mathcal{P}(\mathbb{R}^p \times \mathcal{D})
\cong
\left\{
\left(\pi_k, f_k\right)_{k=1}^K:
\pi_k \ge 0,\;
\sum_{k=1}^K \pi_k = 1,\;
f_k \in \mathcal{P}(\mathbb{R}^p)
\right\},
\]
where $\pi_k=\mathbb{P}(X^{(d)}=k)$ and
$f_k(\cdot)=p(\,\cdot\mid X^{(d)}=k)$. For the true data-generating law
$P^\star$ with density $p^\star$, we use the starred components
\[
\pi_k^\star:= P^\star(X^{(d)} = k),
\qquad
f_k^\star(\cdot):= p^\star(\,\cdot \mid X^{(d)} = k).
\]
In particular, categorical imbalance corresponds to small $\pi_k^\star$ for some $k$.
A generative model must preserve both
marginal probabilities $(\pi_k^\star)$
and conditional densities $(f_k^\star)$.

\subsection{Flow Matching in Continuous Spaces}

Flow Matching (FM) \citep{lipman2022flow}
learns a time-dependent vector field
\[
v_\phi: \mathbb{R}^d \times [0,1] \to \mathbb{R}^d
\]
that induces an ODE
\[
\frac{dY_t}{dt} = v_\phi(Y_t, t),
\]
transporting a base distribution $\mu_0$ to a target distribution $\mu_1^\star$.

Training proceeds by sampling
  \[
Y_1 \sim \mu_1^\star, 
\quad
Y_0 \sim \mu_0,
\quad
t \sim \mathrm{Unif}(0,1),
\] 
defining interpolation
  \[
Y_t = (1-t) Y_0 + t Y_1,
\] 
and minimizing
  \[
\mathcal{L}(\phi)
=
\mathbb{E}
\bigl[
\|v_\phi(Y_t,t) - (Y_1 - Y_0)\|^2
\bigr].
\] 
Equivalently, with   $U=Y_1-Y_0$, the population Flow Matching velocity is
$v^\star(y,t)=\mathbb E[U\mid Y_t=y]$; the loss above is the squared-error
regression objective for this conditional mean.

This formulation assumes $Y_1 \in \mathbb{R}^d$.
Thus categorical variables must be embedded in a continuous space prior to FM training.

\subsection{The Probability Simplex}

Categorical distributions over $\mathcal{D} = \{1, \dots, K\}$ form the probability simplex

\[
\Delta_K
=
\left\{
p \in \mathbb{R}^K:
p_k \ge 0,
\sum_{k=1}^K p_k = 1
\right\}.
\]

Its interior
\[
\Delta_K^\circ = \{p \in \Delta_K: p_k > 0 \ \forall k\}
\]
is a smooth $(K-1)$-dimensional manifold embedded in $\mathbb{R}^K$.

Although $\Delta_K$ is a convex subset of $\mathbb{R}^K$, its intrinsic statistical geometry is not Euclidean. 
The Euclidean metric inherited from $\mathbb{R}^K$ does not coincide with the   Fisher--Rao  metric that governs local behavior of statistical divergence (Section 2.4).

Under one-hot encoding, category $k$ is represented by the vertex $e_k \in \mathbb{R}^K$, where $(e_k)_j = 1\{j=k\}$. 
A Euclidean interpolation between two categories $e_i$ and $e_j$ yields

\[
(1-t)e_i + t e_j,
\quad t \in [0,1],
\]

which corresponds to a mixture distribution assigning probabilities $(1-t)$ and $t$ to categories $i$ and $j$.

Such linear interpolation is natural in Euclidean space, but it does not correspond to the affine paths singled out by the exponential-family structure of categorical distributions.
In particular, e-geodesic, or natural-parameter affine, paths are linear in natural-parameter coordinates rather than in probability coordinates.

This difference becomes visible near the boundary of the simplex. 
For example, in the binary case $K=2$, the simplex reduces to the interval $(0,1)$. 
Euclidean distance between probabilities scales linearly, whereas   the Kullback--Leibler divergence depends on probability ratios such as $p\log(p/q)$. When coordinates approach zero at different rates, the ratio between KL divergence and squared Euclidean distance can diverge. 
Thus,  small Euclidean perturbations  between probability vectors  near the boundary   need not provide uniform quadratic control of  statistical divergence.

\subsection{Information Geometry and the Fisher Metric}

Categorical distributions form an exponential family with density
\[
p_\theta(k)
=
\frac{\exp(\theta_k)}{\sum_{j=1}^K \exp(\theta_j)},
\]
where the natural parameter $\theta \in \mathbb{R}^K$ is defined up to an additive constant.
The associated log-partition function is
\[
A(\theta)
=
\log\!\left(\sum_{j=1}^K \exp(\theta_j)\right).
\]

The Fisher information matrix in natural-parameter coordinates is

\[
\mathcal{I}(\theta)
=
\nabla^2 A(\theta)
=
\mathrm{diag}(p_\theta) - p_\theta p_\theta^\top.
\]

This matrix defines a Riemannian metric on the interior $\Delta_K^\circ$.
For small perturbations $d\theta$, the   Kullback--Leibler  divergence admits the second-order expansion

\[
\mathrm{KL}(p_\theta \,\|\, p_{\theta+d\theta})
=
\frac{1}{2}
d\theta^\top \mathcal{I}(\theta)\, d\theta
+
o(\|d\theta\|^2).
\]

Thus the Fisher metric characterizes the local curvature of the simplex under statistical divergence.

The entries of $\mathcal{I}(\theta)$ depend on the probabilities $p_\theta(k)$.
As $p_\theta(k)\to 0$, the Fisher matrix becomes increasingly anisotropic and loses uniform conditioning near the boundary of the simplex. Thus the constants that connect Euclidean perturbations in natural coordinates to statistical divergence deteriorate near rare-category faces.

\medskip
\noindent\textit{Coordinates and metric.}
Natural-parameter coordinates remove the probability constraint and make
e-geodesic, or natural-parameter affine, paths linear in $\theta$.
However, the intrinsic Riemannian metric on $\Delta_K^\circ$
is given by the Fisher information matrix $\mathcal{I}(\theta)$,
which is generally not the identity matrix.
Thus, although $\theta$-space is an unconstrained Euclidean space as a coordinate representation,
the statistical geometry it represents remains curved.

 The full $K\times K$ Fisher matrix has null vector $\mathbf 1$ because adding
a constant to every natural parameter leaves the distribution unchanged.  On
any compact subset of $\Delta_K^\circ$ bounded away from the boundary,   its
nonzero eigenvalues are  bounded above and below.   Equivalently, after fixing a
reference category, the reduced $(K-1)\times(K-1)$ Fisher matrix is uniformly
positive definite. The corresponding quadratic form  is uniformly equivalent
to the Euclidean norm   on identifiable natural-parameter perturbations.
Consequently, regression objectives based on the Euclidean norm in $\theta$-space locally
approximate the   Fisher--Rao  geometry rather than coincide with it. This local equivalence is made precise in \textbf{Proposition~1} of Section~3.4,
which establishes quadratic control of Kullback--Leibler divergence
by Euclidean perturbations in natural-parameter space.

\subsection{Logit (Natural Parameter) Coordinates}

We use the final category as the reference, denoted $r=K$, and
define reduced natural parameters
\[
\eta_k
=
\log \frac{p_k}{p_r},
\quad
k=1,\dots,K-1.
\]

The mapping
\[
\eta \mapsto p
\]
is smooth and bijective between
$\mathbb{R}^{K-1}$
and
$\Delta_K^\circ$.

In these coordinates, probability ratios vary linearly in $\eta$, and
e-geodesic, or natural-parameter affine, paths correspond to linear paths in $\eta$-space \citep{amari2016information}.
Accordingly, performing regression in $\eta$-space preserves the linear
structure of the exponential family while removing the simplex constraint.

\subsection{Implications for Mixed Transport}
When categorical variables are embedded as one-hot vectors in $\mathbb{R}^K$,
Flow Matching minimizes Euclidean error
\[
\| v - v^\star \|^2
\]
under the flat metric inherited from $\mathbb{R}^K$.
This implicitly treats the probability simplex as a Euclidean subset.

In contrast, embedding categories via natural parameters maps
$\Delta_K^\circ$ to $\mathbb{R}^{K-1}$ through a smooth bijection.
Transport is then performed in logit space.
Although the optimization still uses a Euclidean norm in these coordinates,
natural-parameter affine paths become linear trajectories,
and small perturbations in $\eta$ correspond locally to controlled
perturbations in   Kullback--Leibler  divergence
(Section 2.4).

Thus the difference between one-hot and logit embeddings
lies not in the optimization objective itself,
but in the geometric structure imposed by the coordinate representation.

\section{Geometry of Categorical Relaxations}

In Section 2, we show that categorical distributions lie on the probability simplex equipped with the Fisher--Rao metric, whereas Flow Matching is formulated in Euclidean space. Applying Flow Matching to categorical variables therefore requires choosing an embedding into a Euclidean coordinate system. The choice of embedding determines how Euclidean regression errors translate into statistical discrepancies on the simplex.

In this section we compare two embedding strategies: the standard one-hot Euclidean representation and the logit (natural parameter) representation associated with the categorical exponential family. We show that these two constructions induce fundamentally different geometric behaviors. In particular, Euclidean regression in logit coordinates provides local control of Kullback--Leibler divergence and probability error, while Euclidean distances between probability vectors do not yield uniform control of statistical divergence, especially near the boundary of the simplex. These results clarify the geometric role of categorical embeddings in transport-based generative modeling.

\subsection{One-Hot Euclidean Embedding}
Under one-hot encoding, each category $k \in \{1,\dots,K\}$ is represented by the vertex $e_k \in \mathbb{R}^K$. Flow Matching is then applied in $\mathbb{R}^K$ by treating these vertices as ordinary Euclidean points. Let $\Pi: \mathbb{R}^K \to \{1,\dots,K\}$ denote the decoding map $\Pi(y) = \arg\max_j y_j$. Training minimizes the Euclidean regression loss
\[
\mathcal{L}_{\text{one-hot}} = \mathbb{E}\bigl[ \| v_\phi(Y_t,t) - v^\star(Y_t,t) \|^2 \bigr],
\]
where $Y_t$ interpolates between base samples and one-hot encoded data.

This construction   measures vector-field regression error  in the flat geometry
  of $\mathbb R^K$. The loss $\mathcal{L}_{\text{one-hot}}$ does not directly control a 
statistical divergence between decoded categorical laws; the behavior also depends on the learned
flow and the argmax margins at decoding.

\subsection{Logit Embedding}

In the logit embedding, each categorical observation is mapped to natural parameters of the categorical exponential family. After introducing a small smoothing parameter to ensure strict positivity for $p_k$, define the reduced logits
\[
\eta_k = \log \frac{p_k}{p_K}, 
\qquad k = 1,\dots,K-1,
\]
where $K$ is a fixed reference category. This yields a vector $\eta \in \mathbb{R}^{K-1}$.

The mapping $\eta \mapsto p$ is smooth and bijective between $\mathbb{R}^{K-1}$ and the interior of the simplex $\Delta_K^\circ$.  Flow Matching is carried out in $\eta$-space, and categorical probabilities are  recovered through the softmax map at decoding.

\subsection{Choice of Log-Ratio Coordinates}

We adopt the additive log-ratio (ALR) representation by fixing a reference category. This yields a $(K-1)$-dimensional unconstrained parameterization that coincides with the natural parameterization of the categorical exponential family.

The ALR coordinates depend on the chosen reference category, so they are not
permutation invariant.  Changing the reference category applies an invertible
linear transformation to the log-ratio coordinates and therefore does not
change the represented probability vector. It can still affect finite-sample
optimization because the Euclidean loss and network parameterization are not
invariant to arbitrary linear reparameterizations.  Alternative log-ratio
transformations, such as the centered log-ratio (CLR) representation, provide
symmetric embeddings at the cost of introducing a linear constraint. We use
ALR for   its minimal $(K-1)$-dimensional  exponential-family   parameterization.

\medskip
\noindent\textit{Rare categories in logit space.}

Log-ratio coordinates make rare-category contrasts explicit.
If $p_k \to 0$ while the reference category probability remains bounded away from zero,
then
\[
\eta_k = \log\!\left(\frac{p_k}{p_K}\right) \to -\infty.
\]
Thus categories with small probability are mapped to extreme regions in logit space.
In contrast to one-hot   embeddings, where  all vertices lie at equal Euclidean   distance, 
 logit coordinates induce reference-dependent contrasts whose magnitude grows as probabilities become more imbalanced.

This observation should be read together with the effective-sample-size limitation for rare cells. Logit coordinates make rare/frequent contrasts explicit in the representation, while the available data in a rare cell remain of order $n\pi_k^\star$.

\subsection{Embedding Geometry and Divergence Control}
\label{subsec:key_claims}

The distinction between one-hot and logit embeddings reduces to how Euclidean perturbations in the chosen coordinate system translate into statistical discrepancies on the simplex. In particular, we ask whether small Euclidean errors imply small divergence between the corresponding categorical distributions.

In this section, we demonstrate two results which together clarify the geometric consequences of the embedding choice:

(i) \emph{Local fidelity of the logit embedding.} We show that, in natural parameter coordinates, Euclidean perturbations provide local quadratic control of Kullback--Leibler divergence.

(ii) \emph{Global pathology of Euclidean relaxations.} We show that no analogous global control holds when probability vectors are compared directly in Euclidean norm.

Propositions~\ref{prop:local_kl_control} and \ref{prop:euclidean_mismatch}
are proved below, and Proposition~\ref{prop:prob_error_control}
is proved in Section~5.1.

\begin{proposition}[Local equivalence of KL and reduced natural-parameter error]
\label{prop:local_kl_control}
Let $p_\theta$ denote the categorical distribution with reduced natural parameters
$\theta\in\mathbb{R}^{K-1}$ (with a fixed reference category).
For any compact set $\mathcal{K}\subset \Delta_K^\circ$, there exist constants   $0<c\le C<\infty$  and $r>0$ such that
for all $\theta$ with $p_\theta\in\mathcal{K}$ and all $\delta$ with $\|\delta\|\le r$,
  \[
c\|\delta\|^2
\le
\mathrm{KL}(p_\theta \| p_{\theta+\delta})
\le
C \|\delta\|^2.
\] 
\end{proposition}

\begin{proof}
We work in reduced natural-parameter coordinates $\theta \in \mathbb{R}^{K-1}$,
with the last category fixed as reference. The categorical model is an exponential family with log-partition function
\[
A(\theta) = \log\!\left(1 + \sum_{k=1}^{K-1} e^{\theta_k}\right),
\]
and probabilities
\[
p_\theta(k) = \frac{e^{\theta_k}}{1+\sum_{j=1}^{K-1} e^{\theta_j}},
\quad
p_\theta(K) = \frac{1}{1+\sum_{j=1}^{K-1} e^{\theta_j}}.
\]

  Write
\[
\bar p_\theta=(p_\theta(1),\ldots,p_\theta(K-1))^\top.
\] 
  The  Fisher information matrix in   reduced natural coordinates is
\[
\mathcal{I}_{\mathrm{red}}(\theta)
=
\nabla^2 A(\theta)
=
\mathrm{diag}(\bar p_\theta)-\bar p_\theta\bar p_\theta^\top.
\]
For every $\theta$, this $(K-1)\times(K-1)$ matrix is positive definite.

  Because $\mathcal K$ is compact and lies in the simplex interior, there is a
compact neighborhood $\mathcal K'\subset\Delta_K^\circ$ of $\mathcal K$ and
constants $0<m\le M<\infty$  such that
  \[
mI_{K-1}\preceq \mathcal I_{\mathrm{red}}(\vartheta)
\preceq MI_{K-1}
\qquad\text{whenever }p_\vartheta\in\mathcal K'.
\] 
  Choose $r>0$ so that $p_{\theta+s\delta}\in\mathcal K'$ for all
$p_\theta\in\mathcal K$, $s\in[0,1]$, and $\|\delta\|\le r$.
The integral Taylor remainder for the log-partition function gives
\[
\mathrm{KL}(p_\theta \| p_{\theta+\delta})
=
\int_0^1(1-s)\,
\delta^\top\mathcal I_{\mathrm{red}}(\theta+s\delta)\delta\,ds.
\]

  Consequently,
\[
\frac{m}{2}\|\delta\|^2
\le
\mathrm{KL}(p_\theta \| p_{\theta+\delta})
\le
\frac{M}{2}\|\delta\|^2,
\] 
  which proves the claim.
\end{proof}

\begin{proposition}[Probability error control via softmax Lipschitzness]
\label{prop:prob_error_control}
Let   $S$ be the reference-category softmax map on $\mathbb{R}^{K-1}$.
Then for all $\theta,\tilde\theta\in\mathbb{R}^{K-1}$,
\[
\|S(\theta)-S(\tilde\theta)\|_1 \le \frac{\sqrt{K}}{2}\|\theta-\tilde\theta\|_2.
\]
\end{proposition}

The result follows from the $\ell^2$ Lipschitz bound proved in Section~5.1
together with norm equivalence between $\ell^1$ and $\ell^2$.

\begin{remark}
Proposition~\ref{prop:prob_error_control} provides a quantitative link between
logit estimation error and probability error. 
Since total variation distance between categorical marginals is given by 
$TV(\pi^\star,\widehat\pi)=\frac{1}{2}\|\pi^\star-\widehat\pi\|_1$, 
the $\ell_1$ Lipschitz bound implies that small Euclidean errors in logit space 
yield controlled deviations in categorical probabilities. 
The factor $\sqrt{K}$ arises from norm equivalence between 
$\ell_1$ and $\ell_2$ norms in $\mathbb{R}^K$ and depends only on the number of categories.
\end{remark}

\begin{proposition}[No global Euclidean control of KL on the simplex]
\label{prop:euclidean_mismatch}
There does not exist a constant $C>0$ such that for all 
$p,q \in \Delta_K^\circ$,
\[
KL(p\|q) \le C \|p-q\|_2^2.
\]
\end{proposition}

\begin{proof}
It suffices to construct a counterexample. Consider the binary case $K=2$. 
For $\varepsilon \in (0,1/2)$, define
\[
p_\varepsilon = (\varepsilon, 1-\varepsilon),
\qquad
q_\varepsilon = (\varepsilon^2, 1-\varepsilon^2).
\]

First compute the squared Euclidean distance:
\[
\|p_\varepsilon - q_\varepsilon\|_2^2
= 2(\varepsilon - \varepsilon^2)^2
\sim 2\varepsilon^2
\quad \text{as } \varepsilon \to 0.
\]

Next compute the KL divergence:
\[
KL(p_\varepsilon \| q_\varepsilon)
= \varepsilon \log\!\left(\frac{\varepsilon}{\varepsilon^2}\right)
+ (1-\varepsilon)
\log\!\left(\frac{1-\varepsilon}{1-\varepsilon^2}\right).
\]
The first term equals $\varepsilon \log(1/\varepsilon)$, 
while the second term is $O(\varepsilon)$. 
Hence
\[
KL(p_\varepsilon \| q_\varepsilon)
\sim \varepsilon \log(1/\varepsilon)
\quad \text{as } \varepsilon \to 0.
\]

Therefore,
\[
\frac{KL(p_\varepsilon \| q_\varepsilon)}{\|p_\varepsilon - q_\varepsilon\|_2^2}
\sim
\frac{\varepsilon \log(1/\varepsilon)}{2\varepsilon^2}
=
\frac{\log(1/\varepsilon)}{2\varepsilon}
\;\longrightarrow\; \infty.
\]

Thus no constant $C$ can satisfy the proposed inequality uniformly on $\Delta_K^\circ$.
\end{proof}

 Proposition~\ref{prop:euclidean_mismatch} concerns Euclidean error between
probability vectors. It is not an impossibility result for a one-hot Flow
Matching model with argmax decoding. Its role is narrower: flat
probability-coordinate error alone supplies no uniform quadratic KL control
near the simplex boundary, whereas reduced natural parameters provide local
two-sided control on compact interior sets.

 \subsection{Remarks on Boundary Behavior and Smoothing}

The   two-sided equivalence in  Proposition~\ref{prop:local_kl_control}   is local. 
It holds uniformly  on compact subsets of the   simplex interior. As categorical
probabilities approach zero, the  smallest eigenvalue of the reduced  Fisher
information matrix   can vanish, so the lower quadratic bound and uniform norm
equivalence deteriorate. The global upper control used later for probability
decoding remains valid through softmax Lipschitzness.

In practice, categorical observations are smoothed to ensure strict 
positivity. The smoothing parameter   $\varepsilon$ keeps the encoded prototypes
a fixed distance from the boundary,  where the natural parameterization becomes
unbounded.  Learned trajectories are not automatically confined to a compact
set, so the stability results below use the global Lipschitz property of the
softmax decoder rather than a global KL equivalence.

Finally, the regression objective uses the Euclidean norm in logit 
space. While natural coordinates align e-geodesic, or natural-parameter affine, paths
with linear structure, the induced loss does not coincide exactly with 
the Fisher--Rao metric. A fully geometry-consistent alternative would 
incorporate Fisher-weighted norms, but such modifications are not 
required for the stability analysis developed in later sections. Thus, the
method is information-geometry motivated rather than a Fisher--Rao flow.

\section{Logit Generative Models for Mixed Data}
We now formalize the logit-coordinate representation used by the continuous generative models studied in this paper. The same transformed representation supports both Logit Flow Matching and Logit Diffusion; the two methods differ in the continuous dynamics learned after the categorical variables have been embedded.

\subsection{Problem Setup}

Let
\[
X = (X^{(c)}, X^{(d)})
\]
with
\[
X^{(c)} \in \mathbb{R}^p,
\quad
X^{(d)} \in \{1,\dots,K\}.
\]

We aim to learn a generative model for the true joint distribution
\[
p^\star(x^{(c)}, x^{(d)}).
\]

\subsection{Smoothing of Categorical Observations}

Since logit coordinates require strictly positive probabilities,
we introduce a smoothing parameter $\varepsilon \in (0,1)$.

For an observed category $d \in \{1,\dots,K\}$,
define the smoothed target vector
\[
q^{(\varepsilon)}_k(d)
=
(1-\varepsilon)\mathbf{1}\{k=d\}
+
\varepsilon \frac{1}{K}.
\]
This vector is an encoding of the observed category, not a model for the conditional data law.

This ensures
\[
q^{(\varepsilon)}_k(d) > 0
\quad
\text{for all } k,
\]
so the embedded representation lies strictly in $\Delta_K^\circ$.

\subsection{Logit Transformation}

Use the final category as the reference $r=K$.
Define reduced natural parameters
\[
\eta_k(d)
=
\log \frac{q^{(\varepsilon)}_k(d)}{q^{(\varepsilon)}_r(d)},
\quad
k = 1,\dots,K-1.
\]

The categorical observation $d$ is thus embedded into
\[
\eta(d) \in \mathbb{R}^{K-1}.
\]

The full transformed observation becomes
\[
Y = (X^{(c)}, \eta) \in \mathbb{R}^{p+K-1}.
\]

\subsection{Logit Flow Matching}

Let $\mu_1^\star$ denote the distribution of transformed data $Y$.
Let $\mu_0 = \mathcal{N}(0, I)$ be a base distribution in $\mathbb{R}^{p+K-1}$.

Define interpolation:
  \[
Y_t = (1-t) Y_0 + t Y_1,
\] 
where
  \[
Y_0 \sim \mu_0, 
\quad
Y_1 \sim \mu_1^\star,
\quad
t \sim \mathrm{Unif}(0,1).
\] 

Let   $U=Y_1-Y_0$  be the samplewise regression target. The population velocity is
$v^\star(y,t)=\mathbb E[U\mid Y_t=y]$.

We train a neural network $v_\phi$ by minimizing
\[
\mathcal{L}(\phi)
=
\mathbb{E}
\bigl[
\| v_\phi(Y_t, t) - U \|^2
\bigr].
\]

\subsection{Sampling Procedure}

After training,
generate samples by solving the ODE
\[
\frac{dY_t}{dt} = v_\phi(Y_t, t),
\quad t \in [0,1],
\]
with initial condition $Y_0 \sim \mathcal{N}(0, I)$.

Let the final state be
\[
\widehat{Y}_1 = (\widehat{X}^{(c)}, \hat{\eta}).
\]

\subsection{Decoding to Categorical Variables}

Recover categorical probabilities via softmax:
\[
\widehat{p}_k
=
\frac{\exp(\widehat{\eta}_k)}{1 + \sum_{j=1}^{K-1} \exp(\widehat{\eta}_j)},
\quad
k = 1,\dots,K-1,
\]
and
\[
\widehat{p}_K
=
\frac{1}{1 + \sum_{j=1}^{K-1} \exp(\widehat{\eta}_j)}.
\]

Finally, sample
\[
\widehat{X}^{(d)} \sim \mathrm{Categorical}(\hat{p}).
\]
 This randomized decoder is useful for the probability-level stability analysis
below. A deterministic alternative selects
$\widehat X^{(d)}=\arg\max_k\widehat p_k$. The experiments use this
maximum-probability rule for both one-hot and scaled-logit coordinates.
Randomized softmax decoding incurs an encoding bias of at most
$\varepsilon$ in total variation relative to the unsmoothed categorical law;
maximum-probability decoding recovers every encoded prototype exactly but
requires margin control for a perturbation guarantee.

\begin{algorithm}[H]
\caption{  Logit Flow  Matching for Mixed Data}
\label{alg:logit_fm}
\begin{algorithmic}[1]

\Require Mixed data $(X^{(c)}, X^{(d)})$, smoothing parameter $\varepsilon>0$

\State \textbf{(Embedding)} 
Embed categorical variables via smoothed logit map:
\[
\eta_k = \log \frac{q^{(\varepsilon)}_k}{q^{(\varepsilon)}_r},
\quad k=1,\dots,K-1
\]
Form transformed data $Y = (X^{(c)}, \eta) \in \mathbb{R}^{p+K-1}$.

\State \textbf{(Training)}
Train Flow Matching in $\mathbb{R}^{p+K-1}$ by minimizing
  \[
\mathbb{E}\big[\|v_\phi(Y_t,t) - (Y_1 - Y_0)\|^2\big].
\] 

\State \textbf{(Sampling)}
Sample $Y_0 \sim \mathcal{N}(0,I)$ and solve
\[
\dot Y_t = v_\phi(Y_t,t).
\]

\State \textbf{(Decoding)}
Apply softmax to  the  logit coordinates and   either sample  from the resulting
categorical distribution  or select its maximum-probability category.

\end{algorithmic}
\end{algorithm}
Algorithm~\ref{alg:logit_fm} summarizes the geometry-aware Flow Matching procedure, which modifies the categorical representation while retaining the standard Flow Matching framework.

For multiple categorical variables, the same construction is applied blockwise.
Each categorical column is smoothed, mapped to its own reduced logit coordinate
system, and concatenated with the numerical variables before Flow Matching is
trained. The theoretical analysis below is written for one categorical block to
keep notation transparent; the blockwise product case follows by summing the
corresponding categorical total-variation and conditional transport terms, with
constants depending on the number and cardinalities of the blocks.

\subsection{Logit Diffusion}

Logit Diffusion applies Gaussian diffusion to the same transformed representation
\[
Y=(X^{(c)},\eta)\in\mathbb R^{p+K-1}.
\]
Let $Y_0$ denote the clean transformed data vector. A standard forward noising
process  \citep{song2021score}  has the form
\[
Y_t=\sqrt{\bar\alpha_t}\,Y_0+\sqrt{1-\bar\alpha_t}\,\xi,
\qquad
\xi\sim N(0,I).
\]
A time-conditioned denoising network is trained either to predict the injected noise, the clean target $Y_0$, or an equivalent score parameterization. After reverse sampling, the numerical coordinates are mapped back through the numerical inverse transform and the categorical coordinates are decoded by the softmax map described above.

Thus, Logit Diffusion differs from ordinary Gaussian diffusion only in the coordinate representation used for categorical variables. Through the probability-flow ODE representation, it also fits into the unified vector-field stability framework in the next section: the relevant vector field is the drift induced by the learned score or denoising model.

\section{Stability Analysis for Logit Generative Models}

This section states the stability facts needed for the rate analysis and the mixed-data discrepancy. Detailed derivations are collected in Appendix~\ref{app:stability-proofs}; the main text keeps the assumptions, conclusions, and interpretation needed for the empirical sections.

\subsection{Transport-Based Generative Models}
\begin{definition}[Transport-based generative model]
\label{def:transport_model}
Let $P_0$ be a base distribution on $\mathbb{R}^d$ with finite second moment.
A transport-based generative model is defined by a time-dependent vector field
$f: [0,1] \times \mathbb{R}^d \to \mathbb{R}^d$ and the ODE
\[
\dot X_t = f(t, X_t), \qquad X_0 \sim P_0.
\]
Let $\Phi_t^f$ denote the associated flow map. The generated distribution is
$P_1 = (\Phi_1^f)_\# P_0$.
\end{definition}

\begin{theorem}[Unified stability of transport-based generative models]
\label{thm:unified_stability}
Let $f^\star$ and $\hat f$ be measurable in $t$ and globally $L$-Lipschitz in $x$.
Let $P_0$ have finite second moment, and define
\[
P_1^\star = (\Phi_1^{f^\star})_\# P_0,
\qquad
\widehat P_1 = (\Phi_1^{\hat f})_\# P_0.
\]
Then
\[
W_2(P_1^\star, \widehat P_1)
\le
C
\int_0^1
\left(
\mathbb{E}\| \hat f(t, X_t^\star) - f^\star(t, X_t^\star) \|^2
\right)^{1/2}
dt,
\]
where $X_t^\star = \Phi_t^{f^\star}(X_0)$.
\end{theorem}

The theorem is the common stability mechanism behind the later Flow Matching and diffusion statements: terminal distributional error is controlled by vector-field error accumulated along population trajectories.

\subsection{Stability of the Logit Parameterization}\label{sec:softmax-stability}

Define $S:\mathbb{R}^{K-1}\to \Delta_K^\circ$ by
\[
S(\eta)_k
=
\frac{\exp(\eta_k)}{1+\sum_{j=1}^{K-1}\exp(\eta_j)},
\quad k=1,\dots,K-1,
\qquad
S(\eta)_K
=
\frac{1}{1+\sum_{j=1}^{K-1}\exp(\eta_j)}.
\]
This takes category $K$ as the reference category.

\begin{theorem}[Global Lipschitz bound of softmax]
\label{thm:softmax-lipschitz}
For all $\eta,\tilde{\eta}\in\mathbb{R}^{K-1}$,
\[
\|S(\eta)-S(\tilde{\eta})\|_2
\le
\frac{1}{2}\|\eta-\tilde{\eta}\|_2.
\]
Consequently,
\[
\|S(\eta)-S(\tilde{\eta})\|_1
\le
\frac{\sqrt{K}}{2}\|\eta-\tilde{\eta}\|_2.
\]
\end{theorem}

 To see the reduced-coordinate bound, let $E:\mathbb R^{K-1}\to\mathbb R^K$
insert a zero in the reference coordinate. The Jacobian of $S$ is
$\{\mathrm{diag}(p)-pp^\top\}E$, so its operator norm is at most $1/2$ by
Lemma~\ref{lem:softmax-spectral-bound}.  Thus logit-coordinate error gives direct control of categorical probability error.
 When the trained coordinate is the scaled logit $z=\eta/a_K$, decoding is through $S(a_K z)$. The same stability statements therefore hold with the softmax Lipschitz constant multiplied by $a_K$; this factor is absorbed in $C_{\mathrm{sm}}$ in the corollaries below. The theorem controls probability decoding. Hard decoding rules such as argmax require separate margin conditions or empirical diagnostics, which is why the experiments report categorical total-variation and rare-cell diagnostics directly.

\subsection{Local KL Geometry and Mixed Fidelity}

Natural-parameter coordinates also give local control of intrinsic divergence. For categorical exponential families,
  \[
\mathrm{KL}(p_\theta \| p_{\theta+\delta})
=
\frac{1}{2} \delta^\top \mathcal{I}_{\mathrm{red}}(\theta)\delta
+
o(\|\delta\|^2),
\] 
where   $\mathcal{I}_{\mathrm{red}}(\theta)$  is the Fisher information matrix  in
the $(K-1)$ reduced natural coordinates; see Lemma~\ref{lem:kl-expansion}.

Let $X=(X^{(c)},X^{(d)})$ with $X^{(c)}\in\mathbb{R}^p$ and $X^{(d)}\in\{1,\dots,K\}$. For the true and generated laws $P^\star$ and $\widehat P$, write
\[
\pi_k^\star = P^\star(X^{(d)}=k),
\qquad
\widehat\pi_k = \widehat P(X^{(d)}=k),
\]
and
\[
P_k^{(c),\star} = \mathcal{L}_{P^\star}(X^{(c)}\mid X^{(d)}=k),
\qquad
\widehat P_k^{(c)}
=
\mathcal{L}(\widehat X^{(c)}\mid \widehat X^{(d)}=k).
\]
We use the mixed discrepancy
\begin{equation}
\label{eq:mixed-distance}
\mathcal{D}_{\mathrm{mix}}(P^\star,\widehat P)
:=
TV(\pi^\star,\widehat\pi)
+
\sum_{k=1}^K \pi_k^\star\, W_2\!\left(P_k^{(c),\star},\widehat P_k^{(c)}\right),
\end{equation}
where $\pi^\star=(\pi_1^\star,\ldots,\pi_K^\star)$ and
$TV(\pi^\star,\widehat\pi)=\frac{1}{2}\sum_{k=1}^K |\pi_k^\star-\widehat\pi_k|$.
This discrepancy separates two errors that have different statistical origins: misallocation of categorical mass and distortion of the conditional continuous law.
In the population definition, cells with $\pi_k^\star=0$ contribute zero to the weighted sum. The displayed discrepancy is used under the regularity convention that $\widehat P_k^{(c)}$ is defined for cells with positive true mass, as occurs for population softmax decoding with positive category probabilities. Finite generated samples can have empty empirical cells; empirical implementations should report the corresponding support convention when conditional Wasserstein summaries are computed.

A joint product-space optimal transport distance can also be defined using
\[
c_\lambda\big((x,d),(x',d')\big)
=
\|x-x'\|_2^2+
\lambda \mathbf{1}\{d\neq d'\},
\qquad \lambda>0.
\]
The decomposed discrepancy in \eqref{eq:mixed-distance} avoids choosing the calibration parameter $\lambda$ and aligns directly with the logit and conditional-continuous stability bounds below.

\subsection{Flow-Matching Fidelity Bounds and Extension to Mixed Data}
\label{sec:stability}

Let $v^\star, \hat v: [0,1]\times\mathbb{R}^p \to \mathbb{R}^p$ be measurable in $t$.

\begin{assumption}[Continuous Flow Regularity]
\label{ass:continuous_flow}
Both $v^\star$ and $\hat v$ are globally $L$-Lipschitz in $x$, uniformly in $t\in[0,1]$; the initial law has finite second moment; and
\[
\int_0^1
\mathbb{E}
\|\hat v(t,X_t^\star) - v^\star(t,X_t^\star)\|^2
\, dt
< \infty,
\]
where $X_t^\star = \Phi_t^{v^\star}(X_0)$ and $X_0\sim P_0$.
\end{assumption}

\begin{theorem}[Wasserstein Stability for Flow Matching]
\label{thm:fm_w2_stability}
Suppose Assumption~\ref{ass:continuous_flow} holds. Then
\begin{equation}
\label{eq:fm_w2_bound}
W_2(P_1^\star, \widehat P_1)
\le
 e^{L}
\int_0^1
\Big(
\mathbb{E}\|\hat v(t,X_t^\star) - v^\star(t,X_t^\star)\|^2
\Big)^{1/2}
dt,
\end{equation}
where $P_1^\star=(\Phi_1^{v^\star})_\#P_0$ and $\widehat P_1=(\Phi_1^{\hat v})_\#P_0$.
\end{theorem}

\begin{proof}
Use the synchronous coupling generated by the same $X_0$, subtract the two ODEs, and apply Gronwall's inequality followed by Minkowski's inequality. Appendix~\ref{app:stability-proofs} gives the details.
\end{proof}

The   next two corollaries concern a componentwise generator with a logit flow
for categorical mass and separate continuous flows for the class-conditional
laws. This formulation is an idealized version of the block-conditional
factorization used in the experiments. For the flat joint model,
Theorem~\ref{thm:unified_stability} directly controls Wasserstein error for the
full  transformed row $Y=(X^{(c)},\eta)$. Projection gives marginal coordinate
bounds, but it does not by itself give the class-conditional transport bounds
assumed below; obtaining those bounds for a flat joint neural flow requires
additional conditional-stability or margin assumptions.

\begin{corollary}[Logit marginal fidelity for Flow Matching]
\label{cor:tv_from_logit_fm}
Let $Z_t \in \mathbb{R}^{K-1}$ denote logit coordinates and let $S$ be the softmax map. Suppose the population and learned logit dynamics are
\[
\dot Z_t = v_Z^\star(t,Z_t),
\qquad
\dot{\widehat Z}_t = \hat v_Z(t,\widehat Z_t),
\]
with common base initialization $Z_0\sim P_{0,Z}$. Let
  \[
\pi^{\star,(\varepsilon)}:= \mathbb{E}[S(Z_1^\star)],
\qquad
\widehat \pi:= \mathbb{E}[S(\widehat Z_1)].
\] 
If  the population endpoint is the smoothed encoding of a categorical variable
with unsmoothed marginal $\pi^\star$, then
$\pi^{\star,(\varepsilon)}=(1-\varepsilon)\pi^\star+\varepsilon u_K$, where $u_K$ is
uniform on the $K$ categories, and
$TV(\pi^\star,\pi^{\star,(\varepsilon)})\le\varepsilon$.
If  $\hat v_Z$ is globally $L_Z$-Lipschitz in $z$ and
$\|S(z)-S(z')\|_1\le C_{\mathrm{sm}}\|z-z'\|_2$, then
  \begin{equation}
\label{eq:tv_logit_fm_bound}
TV(\pi^\star,\widehat\pi)
\le
\varepsilon
+
\frac{C_{\mathrm{sm}}}{2}\,
e^{L_Z}
\int_0^1
\left(
\mathbb{E}\big\|
\hat v_Z(t,Z_t^\star) - v_Z^\star(t,Z_t^\star)
\big\|_2^2
\right)^{1/2}
dt.
\end{equation} 
\end{corollary}

\begin{corollary}[Mixed fidelity for Flow Matching]
\label{cor:mixed_fidelity_fm}
Consider mixed data $(X^{(c)},X^{(d)})$ with $X^{(c)}\in\mathbb{R}^p$ and $X^{(d)}\in\{1,\dots,K\}$. Let the logit component follow the dynamics in Corollary~\ref{cor:tv_from_logit_fm}. For each class $k$, suppose the conditional continuous component evolves according to
\[
\dot X_{t,k}^{(c)} = v_k^\star(t,X_{t,k}^{(c)}),
\qquad
\dot{\hat X}_{t,k}^{(c)} = \hat v_k(t,\hat X_{t,k}^{(c)}),
\]
with common base initialization and $\hat v_k$ globally $L_k$-Lipschitz. Then
  \begin{equation}
\label{eq:mixed_fidelity_fm_bound}
\begin{aligned}
\mathcal{D}_{\mathrm{mix}}(P^\star,\widehat P)
&\le
\varepsilon
+
\frac{C_{\mathrm{sm}}}{2}\,e^{L_Z}
\int_0^1
\left(
\mathbb{E}\big\|
\hat v_Z(t,Z_t^\star)-v_Z^\star(t,Z_t^\star)
\big\|_2^2
\right)^{1/2}
dt \\
&\quad+
\sum_{k=1}^K
\pi_k^\star\, e^{L_k}
\int_0^1
\left(
\mathbb{E}\big\|
\hat v_k(t,X_{t,k}^\star)-v_k^\star(t,X_{t,k}^\star)
\big\|_2^2
\right)^{1/2}
dt.
\end{aligned}
\end{equation} 
\end{corollary}

Equation~\eqref{eq:mixed_fidelity_fm_bound} is the   componentwise  bridge from
learned velocity fields to   mixed-data  fidelity. It separates the
logit marginal term, including smoothing bias,  from the class-conditional
continuous terms; the rate analysis uses the same componentwise contract.

\subsection{Product-Space and Diffusion Consequences}

For comparison with a single product-space transport metric, define $W_{c_\lambda}$ using the cost above.

\begin{assumption}[Uniform conditional second moments]
\label{ass:second_moment}
For   each  cell with positive   mass  under $P^\star$ or $\widehat P$, the   continuous conditional laws  satisfy
\[
\mathbb{E}\!\left(\|X^{(c)}\|_2^2\mid X^{(d)}=k\right)\le M,
\qquad
\mathbb{E}\!\left(\|\widehat X^{(c)}\|_2^2\mid \widehat X^{(d)}=k\right)\le M
\]
for some $M<\infty$.
\end{assumption}

\begin{lemma}[Joint transport upper bound]
\label{lem:joint_transport_upper}
Suppose Assumption~\ref{ass:second_moment} holds. Then
\[
W_{c_\lambda}^2(P^\star,\widehat P)
\le
\sum_{k=1}^K
\pi_k^\star W_2^2\!\left(P_k^{(c),\star},\widehat P_k^{(c)}\right)
+
(\lambda + 4M)\,TV(\pi^\star,\widehat\pi).
\]
\end{lemma}

\begin{theorem}[Product-space stability for mixed Flow Matching]
\label{thm:joint_ot_stability}
Assume the hypotheses of Corollary~\ref{cor:mixed_fidelity_fm} and Assumption~\ref{ass:second_moment}. Then $W_{c_\lambda}^2(P^\star,\widehat P)$ is bounded by the right side of Lemma~\ref{lem:joint_transport_upper} after substituting the classwise $W_2$ bounds from Theorem~\ref{thm:fm_w2_stability} and the total-variation bound from Corollary~\ref{cor:tv_from_logit_fm}.
\end{theorem}

\begin{proposition}[Diffusion probability-flow stability]
\label{prop:diffusion_stability}
Let $s^\star(t,x)=\nabla\log p_t^\star(x)$ be the population score and let $\hat s$ be an estimator. For the probability flow drifts
\[
f^\star(t,x)=\mu(t,x)-\tfrac12\sigma^2(t)s^\star(t,x),
\qquad
\hat f(t,x)=\mu(t,x)-\tfrac12\sigma^2(t)\hat s(t,x),
\]
assume $f^\star$ and $\hat f$ are globally Lipschitz in $x$ uniformly over $t\in[0,1]$ and that the initial laws have finite second moments. Then
\[
W_2(P_1^\star,\widehat P_1)
\le
C
\int_0^1
\bigl(
\mathbb{E}\|\hat s(t,X_t^\star)-s^\star(t,X_t^\star)\|^2
\bigr)^{1/2}
\,dt,
\]
for a constant $C$ depending on the Lipschitz constants and $\sigma(t)$.
\end{proposition}

\begin{corollary}[Logit-diffusion fidelity]
\label{cor:mixed_diffusion_stability}
  Suppose the  logit and class-conditional   ODEs satisfy the  conditions of Proposition~\ref{prop:diffusion_stability}. The mixed discrepancy then  obeys the same structural bound as \eqref{eq:mixed_fidelity_fm_bound}, with the Flow Matching velocity errors replaced by probability-flow drift errors. Equivalently, using
$\widehat f-f^\star=-\frac{1}{2}\sigma^2(t)(\widehat s-s^\star)$, the bound can be written directly in terms of score-estimation error.
\end{corollary}

These diffusion consequences are included to show that the logit-coordinate representation is algorithm-agnostic. The empirical comparisons later use the same representation principle within Flow Matching and diffusion families.

\section{Statistical Convergence Rates for Logit Generative Models}

This section gives the rate calculation needed to interpret the rare-cell
behavior in the experiments. The argument combines the stability bound in
Section~\ref{sec:stability} with standard nonparametric regression rates for
the learned vector fields.  Related Flow Matching theory derives nearly minimax
distributional rates under a different set of assumptions
\citep{fukumizu2024flowmatching}.

  We state one explicit sieve result for the componentwise model used in
Corollary~\ref{cor:mixed_fidelity_fm}. This avoids treating the calculation as
an implementation-level guarantee for the flat joint neural generator.

 \subsection{Flow Matching as Nonparametric Regression}

Let $Y_1\sim\mu_1^\star$ be the transformed data distribution,   $Y_0\sim\mu_0$  the base distribution, and $t\sim\mathrm{Unif}(0,1)$ independent. Define
  \[
Y_t=(1-t)Y_0+tY_1,
\qquad
U=Y_1-Y_0,
\qquad
X=(Y_t,t).
\] 
The population Flow Matching risk is $\mathcal L(v)=\mathbb E\|v(X)-U\|^2$.

\begin{proposition}[Population Minimizer]
\label{prop:pop-min}
The function $v^\star(x)=\mathbb E[U\mid X=x]$ is a minimizer of $\mathcal L(v)$ over square-integrable functions.
\end{proposition}

This is the standard $L^2$ projection property of conditional expectation; Appendix~\ref{app:regression} gives the proof.

\begin{assumption}[Regularity for velocity regression]
\label{ass:velocity_regularity}
Let $d_Z=K-1$. The population logit velocity $v_Z^\star(t,z)$ is $\alpha_Z$-Hölder on $[0,1]\times\mathcal Z\subset\mathbb R^{1+d_Z}$, and each class-conditional continuous velocity $v_k^\star(t,x)$ is $\alpha_c$-Hölder on $[0,1]\times\mathcal X\subset\mathbb R^{1+p}$. The domains $\mathcal Z$ and $\mathcal X$ are bounded, the conditional variance in the regression problem is uniformly finite, and the estimators use $m$-dimensional linear sieve classes $\mathcal F_m\subset L^2(P_X)$.
\end{assumption}

\begin{lemma}[Sieve regression rate]
\label{lem:sieve}
Under Assumption~\ref{ass:velocity_regularity}, there exists a constant $C>0$ such that
\[
\mathbb E\|\hat v_m(X)-v^\star(X)\|^2
\le
C\left(m^{-2\alpha/d}+\frac{m}{n}\right),
\]
where $d$ denotes the input dimension of $X$. Choosing $m\asymp n^{d/(2\alpha+d)}$ gives $\mathbb E\|\hat v_m-v^\star\|^2\lesssim n^{-2\alpha/(2\alpha+d)}$.
\end{lemma}

The proof follows standard linear least-squares sieve bounds under finite
conditional variance and Hölder approximation  \citep{Tsybakov2009}; see
Appendix~\ref{app:sieve}.

\begin{assumption}[Bounded Density Ratio]
\label{ass:density}
The distribution of states visited by the population flow is absolutely continuous with respect to the regression input distribution, with Radon--Nikodym derivative $d\mu_{\mathrm{path}}/d\mu_{\mathrm{train}}\le\kappa$.
\end{assumption}

\begin{lemma}[Bridge Inequality]
\label{lem:bridge}
Under Assumption~\ref{ass:density},
\[
\mathbb E_{\mathrm{path}}\|\hat v-v^\star\|^2
\le
\kappa\,\mathbb E_{\mathrm{train}}\|\hat v-v^\star\|^2.
\]
\end{lemma}

This is a change-of-measure step; Appendix~\ref{app:bridge} gives the proof.

\subsection{Main Result: Mixed Fidelity Rate}

We now combine Lemma~\ref{lem:sieve}, Lemma~\ref{lem:bridge}, and
Corollary~\ref{cor:mixed_fidelity_fm}. The result is a population-level
regularity statement for bounded transformed domains and idealized
componentwise velocity regressions. A finite-sample guarantee for the exact
neural-network training and numerical sampling procedures used in the
experiments would additionally require localization, optimization, and
discretization analyses.

\begin{theorem}[Nonparametric Rate for Componentwise Mixed Flow Matching]
\label{thm:mixed-rate}
Suppose Assumptions~\ref{ass:velocity_regularity} and~\ref{ass:density}
hold, and suppose that the population and learned velocity fields are
uniformly Lipschitz on the bounded domains over $t\in[0,1]$. Then
\[
\mathcal{D}_{\mathrm{mix}}(P^\star,\widehat P)
\leq
\varepsilon
+
C_1 n^{-\alpha_Z/(2\alpha_Z+d_Z+1)}
+
\sum_{k=1}^K
\pi_k^\star C_{2,k}(n\pi_k^\star)^{-\alpha_c/(2\alpha_c+p+1)},
\]
where $d_Z=K-1$ and
$\pi_k^\star=P^\star(X^{(d)}=k)$.
\end{theorem}

The theorem separates categorical and conditional continuous estimation.
The categorical logit component uses all $n$ observations, whereas the
conditional continuous component in cell $k$ has effective sample size
$n\pi_k^\star$. Its contribution to $\mathcal{D}_{\mathrm{mix}}$ is then weighted
by the population mass $\pi_k^\star$. Thus, $\mathcal{D}_{\mathrm{mix}}$
summarizes population-level fidelity, while
$\mathcal{D}_{\mathrm{bal}}$ and the rare-cell diagnostics used in the
simulation more directly expose errors in low-probability cells. Uniform
finite-sample statements over rare cells would additionally require lower
bounds on the relevant cell probabilities and concentration control for
the empirical cell counts.

\begin{remark}[Smoothing and decoding]
For fixed $\varepsilon$, the theorem controls fidelity to the unsmoothed
data law up to the additive encoding bias. Consistency under randomized
softmax decoding therefore requires a sequence
$\varepsilon=\varepsilon_n\to0$. However, the scaled-logit decoding
constant involves
\[
a_K
=
\log\left\{
\frac{1-\varepsilon+\varepsilon/K}{\varepsilon/K}
\right\},
\]
which diverges as $\varepsilon\to0$. The theorem therefore does not specify
an optimal smoothing schedule, and any rate statement involving
$\varepsilon_n$ requires the regularity and stability constants to remain
uniform along that sequence. The maximum-probability decoder used in the
experiments has no prototype-level smoothing bias, but its perturbation
analysis requires a margin condition and is outside the scope of the
theorem.
\end{remark}

\begin{remark}[Dominating component of the rate]
Suppose that $K$ is fixed, the constants in Theorem~\ref{thm:mixed-rate}
do not depend on $n$, and
\[
\min_{1\leq k\leq K}\pi_k^\star\geq \pi_{\min}^\star>0.
\]
Define
\[
\beta_Z
=
\frac{\alpha_Z}{2\alpha_Z+d_Z+1},
\qquad
\beta_c
=
\frac{\alpha_c}{2\alpha_c+p+1}.
\]
Then
\[
\sum_{k=1}^K
\pi_k^\star C_{2,k}(n\pi_k^\star)^{-\beta_c}
=
O(n^{-\beta_c}),
\]
so the statistical rate is determined by the slower of
$n^{-\beta_Z}$ and $n^{-\beta_c}$.

If
\[
\beta_Z<\beta_c,
\qquad\text{equivalently}\qquad
\alpha_Z(p+1)<\alpha_c(d_Z+1),
\]
then the categorical logit term dominates. If instead
\[
\beta_c<\beta_Z,
\qquad\text{equivalently}\qquad
\alpha_c(d_Z+1)<\alpha_Z(p+1),
\]
then the conditional continuous term dominates. When
$\beta_Z=\beta_c$, both components contribute at the same order.

Let
\[
\beta_*=\min\{\beta_Z,\beta_c\}.
\]
If the smoothing parameter is chosen so that
\[
\varepsilon_n=O(n^{-\beta_*}),
\]
and the regularity and stability constants remain uniform over this
sequence, then
\[
\mathcal{D}_{\mathrm{mix}}(P^\star,\widehat P)
=
O_p(n^{-\beta_*}).
\]
More explicitly,
\[
\mathcal{D}_{\mathrm{mix}}(P^\star,\widehat P)
=
\begin{cases}
O_p\!\left(
n^{-\alpha_Z/(2\alpha_Z+d_Z+1)}
\right),
&
\alpha_Z(p+1)<\alpha_c(d_Z+1),
\\[6pt]
O_p\!\left(
n^{-\alpha_c/(2\alpha_c+p+1)}
\right),
&
\alpha_c(d_Z+1)<\alpha_Z(p+1),
\end{cases}
\]
with both terms contributing at the same order in the equality case.
If $\varepsilon_n=o(n^{-\beta_*})$, the smoothing bias is asymptotically
negligible. If some $\pi_k^\star$ decrease with $n$, the cellwise
effective-sample-size terms must instead be retained explicitly.
\end{remark}

\section{Simulation Study}

This section presents controlled simulation experiments designed to evaluate the role of categorical representation in generative modeling for mixed continuous--categorical data. The goal is to separate the effect of the categorical coordinate system from the effect of the surrounding generative algorithm. We therefore compare one-hot and scaled-logit embeddings within the same model families: Flow Matching and Gaussian diffusion.

Unlike a purely categorical marginal experiment, the simulation below includes mixed conditional dependence. The data-generating process contains several categorical factors, a rare joint categorical cell, and continuous variables whose conditional distribution depends on the full categorical state. This design reflects the main difficulty of mixed tabular generation: a model must preserve both joint categorical mass and the numerical distribution within categorical cells.

\subsection{Data-Generating Process}

Each observation has the form
\[
X=(X_{\mathrm{num}},A,B,Y),
\]
where
\[
X_{\mathrm{num}}\in\mathbb R^4,\qquad
A\in\{0,1,2,3\},\quad
B\in\{0,1,2\},\quad
Y\in\{0,1\}.
\]
We combine the categorical variables into the joint state
\[
G=(A,B,Y),
\]
so that the categorical block has $K=4\times 3\times 2=24$ possible states. We designate one rare joint cell,
\[
g_{\mathrm{rare}}=(A=3,B=2,Y=1),
\]
and set
\[
\mathbb P(G=g_{\mathrm{rare}})=\rho,\qquad \rho\in\{0.05,0.01\}.
\]
The remaining categorical probabilities follow a long-tailed distribution with dependence among $A$, $B$, and $Y$; the full specification is given in Appendix~B.

Conditional on the joint categorical state, the numerical variables follow a Gaussian distribution,
\[
X_{\mathrm{num}}\mid G=g\sim N(\mu_g,\Sigma_g).
\]
Both the mean vector and covariance matrix depend on the full categorical state $g=(a,b,y)$ through additive and interaction effects. Thus, the conditional numerical distribution changes across categorical cells. This construction evaluates whether a generator preserves both $\mathbb P(G)$ and $\mathbb P(X_{\mathrm{num}}\mid G)$.

For each imbalance regime, we generate $n_{\mathrm{train}}=60{,}000$ training samples and $n_{\mathrm{test}}=20{,}000$ test samples independently from the data-generating distribution. The generated sample size is also 20,000. All results are aggregated across ten random seeds.

\subsection{Categorical Representations and Methods}

We compare two continuous representations of the joint categorical state $G$. The one-hot representation maps each state $g\in\{1,\ldots,K\}$ to the vertex $e_g\in\mathbb R^K$, and generated categorical values are decoded by the argmax rule. The scaled-logit representation first smooths the categorical observation,
\[
p_j^{(\varepsilon)}(g)
=(1-\varepsilon)\mathbf 1\{j=g\}+\frac{\varepsilon}{K},
\qquad j=1,\ldots,K,
\]
then applies additive log-ratio coordinates with a fixed reference category $r$,
\[
\eta_j(g)=\log\frac{p_j^{(\varepsilon)}(g)}{p_r^{(\varepsilon)}(g)},
\qquad j\ne r.
\]
The logits are normalized by the prototype magnitude
\[
a_K=\log\frac{1-\varepsilon+\varepsilon/K}{\varepsilon/K},
\qquad
z_j(g)=\eta_j(g)/a_K.
\]
In implementation, the reference category is chosen as a high-frequency category from the training data for numerical stability and is fixed before model fitting. This is a coordinate convention, not a separate model component.
 Generated scaled logits are mapped through the inverse softmax and decoded by
maximum-probability selection, matching the argmax rule used for the one-hot
representation.
 Because ALR coordinates depend on the reference category, finite-sample
training can in principle be sensitive to this choice; systematic
reference-category sensitivity analysis is left for future work.

We use a $2\times 2$ comparison that varies the categorical representation and the generative algorithm:
\[
\text{representation}\in\{\text{one-hot},\text{scaled-logit}\},
\qquad
\text{algorithm}\in\{\text{Flow Matching},\text{diffusion}\}.
\]
This gives four methods: OneHot-FM, Logit-FM, OneHot-Diffusion, and Logit-Diffusion. All methods are flat joint generators in the simulation and model the transformed vector $(Z_{\mathrm{num}},Z_{\mathrm{cat}})$ jointly. This choice keeps the comparison focused on representation, with conditional factorization held fixed. Within each algorithmic family, the one-hot and logit variants use the same architecture, optimizer, batch size, and sampling procedure. We therefore interpret the experiments as two within-family representation comparisons, with no general ranking of Flow Matching against diffusion.

\subsection{Evaluation Metrics}

Let $\pi_g^\star=\mathbb P(G=g)$ denote the true joint categorical distribution and let $\widehat\pi_g$ be the generated empirical distribution. We report categorical total variation $TV(\pi^\star,\widehat\pi)=\|\widehat\pi-\pi^\star\|_1/2$, the primary mixed discrepancy
\[
\mathcal{D}_{\mathrm{mix}}(P^\star,\widehat P)
=TV(\pi^\star,\widehat\pi)
+\sum_{g=1}^K \pi_g^\star
W_2\!\left(P^{\mathrm{num},\star}_g,\widehat P^{\mathrm{num}}_g\right),
\]
and the balanced diagnostic
\[
\mathcal{D}_{\mathrm{bal}}(P^\star,\widehat P)
=TV(\pi^\star,\widehat\pi)
+\frac{1}{K}\sum_{g=1}^K
W_2\!\left(P^{\mathrm{num},\star}_g,\widehat P^{\mathrm{num}}_g\right).
\]
The weighted version reflects the population mixed distribution, while the balanced version makes low-probability cell distortions more visible. For the designated rare cell $g_{\mathrm{rare}}$, we also report generated rare-cell mass, absolute rare-cell error, and rare-cell conditional Wasserstein error. Finally, we report a factor-level conditional label diagnostic, $\mathrm{CondErr}_{Y\mid AB}$, and a two-sample AUC for distinguishing real from generated samples.

\subsection{Results}

\medskip
\noindent\textit{Flow Matching results.}
Table~\ref{tab:sim_fm} reports the mean and standard deviation over ten random seeds. At the moderate imbalance level $\rho=0.05$, the rare joint cell has expected training support $n_{\mathrm{train}}\rho=3{,}000$, and the two Flow Matching representations are essentially tied. Under severe imbalance, $\rho=0.01$, the expected rare-cell support falls to 600. In this setting, Logit-FM reduces categorical total variation, $\mathcal{D}_{\mathrm{mix}}$, $\mathcal{D}_{\mathrm{bal}}$, conditional label error, and two-sample AUC relative to OneHot-FM.

\begin{table}[ht]
\centering
\small
\caption{Flow Matching simulation. Here $\rho$ is the rare-cell mass; TV compares joint categorical--label masses; $\mathcal{D}_{\mathrm{mix}}$ combines TV with mass-weighted conditional numerical $W_2$; $\mathcal{D}_{\mathrm{bal}}$ weights cells equally; $\mathrm{CondErr}_{Y\mid AB}$ measures error in $P(Y\mid A,B)$; and two-sample AUC has ideal value 0.5. Entries are mean $\pm$ SD over ten seeds. Lower is better, except AUC is better closer to 0.5. Logit-FM improves all five means at $\rho=0.01$ and is nearly tied with OneHot-FM at $\rho=0.05$.}
\label{tab:sim_fm}
\resizebox{\textwidth}{!}{%
\begin{tabular}{llccccc}
\hline
$\rho$ & Method & $TV(\pi^\star,\widehat\pi)$ & $\mathcal{D}_{\mathrm{mix}}$ & $\mathcal{D}_{\mathrm{bal}}$ & $\mathrm{CondErr}_{Y\mid AB}$ & AUC\\
\hline
0.05 & OneHot-FM & $0.040\pm0.003$ & $0.158\pm0.010$ & $0.190\pm0.010$ & $0.032\pm0.009$ & $0.518\pm0.003$\\
0.05 & Logit-FM & $0.039\pm0.006$ & $0.158\pm0.011$ & $0.190\pm0.011$ & $0.032\pm0.007$ & $0.517\pm0.006$\\
0.01 & OneHot-FM & $0.044\pm0.007$ & $0.165\pm0.008$ & $0.203\pm0.008$ & $0.036\pm0.008$ & $0.523\pm0.005$\\
0.01 & Logit-FM & $0.038\pm0.007$ & $0.157\pm0.007$ & $0.195\pm0.007$ & $0.032\pm0.008$ & $0.517\pm0.006$\\
\hline
\end{tabular}}
\end{table}

\begin{table}[ht]
\centering
\small
\caption{Flow Matching rare-cell diagnostics for $g_{\mathrm{rare}}=(A=3,B=2,Y=1)$, whose true mass is $\rho$. The columns report generated mass, absolute mass error, and conditional numerical $W_2$ within $g_{\mathrm{rare}}$. Entries are mean $\pm$ SD over ten seeds; lower error is better. At $\rho=0.01$, Logit-FM has smaller mean mass error, while the $W_2$ means are similar relative to their SDs.}
\label{tab:sim_fm_rare}
\begin{tabular}{llccc}
\hline
$\rho$ & Method & $\widehat\pi_{g_{\mathrm{rare}}}$ & $|\widehat\pi_{g_{\mathrm{rare}}}-\pi_{g_{\mathrm{rare}}}^\star|$ & Rare-cell $W_2$\\
\hline
0.05 & OneHot-FM & $0.049\pm0.004$ & $0.003\pm0.002$ & $0.121\pm0.025$\\
0.05 & Logit-FM & $0.051\pm0.005$ & $0.004\pm0.003$ & $0.127\pm0.022$\\
0.01 & OneHot-FM & $0.010\pm0.003$ & $0.002\pm0.002$ & $0.284\pm0.077$\\
0.01 & Logit-FM & $0.011\pm0.002$ & $0.001\pm0.001$ & $0.281\pm0.088$\\
\hline
\end{tabular}
\end{table}

The rare-cell diagnostics in Table~\ref{tab:sim_fm_rare} are more variable because they depend on the lowest-support joint categorical cell. At $\rho=0.01$, the rare-cell mass error is directionally smaller for Logit-FM, while the rare-cell conditional Wasserstein values are nearly tied relative to their standard deviations. The broader separation under severe imbalance is therefore read from the combined $\mathcal{D}_{\mathrm{mix}}$, $\mathcal{D}_{\mathrm{bal}}$, categorical-mass, and conditional-label diagnostics rather than from rare-cell Wasserstein alone.

\medskip
\noindent\textit{Diffusion results.}
Table~\ref{tab:sim_diffusion} reports the corresponding diffusion comparison. Within the diffusion family, the scaled-logit representation has slightly lower reported mean mixed-distribution discrepancies than the one-hot representation in both imbalance regimes. At $\rho=0.05$, Logit-Diffusion reduces $\mathcal{D}_{\mathrm{mix}}$ from 0.398 to 0.392 and $\mathcal{D}_{\mathrm{bal}}$ from 0.408 to 0.400. At $\rho=0.01$, it reduces $\mathcal{D}_{\mathrm{mix}}$ from 0.390 to 0.383 and $\mathcal{D}_{\mathrm{bal}}$ from 0.399 to 0.392. The conditional label error is essentially tied, so the diffusion gains are driven mainly by mixed-distribution fidelity.

\begin{table}[ht]
\centering
\small
\caption{Gaussian diffusion simulation. Here $\rho$ is the rare-cell mass; TV compares joint categorical--label masses; $\mathcal{D}_{\mathrm{mix}}$ combines TV with mass-weighted conditional numerical $W_2$; $\mathcal{D}_{\mathrm{bal}}$ weights cells equally; $\mathrm{CondErr}_{Y\mid AB}$ measures error in $P(Y\mid A,B)$; and two-sample AUC has ideal value 0.5. Entries are mean $\pm$ SD over ten seeds. Lower is better, except AUC is better closer to 0.5. Logit-Diffusion lowers both mixed discrepancies and AUC at each $\rho$; TV and conditional-label error are nearly tied.}
\label{tab:sim_diffusion}
\resizebox{\textwidth}{!}{%
\begin{tabular}{llccccc}
\hline
$\rho$ & Method & $TV(\pi^\star,\widehat\pi)$ & $\mathcal{D}_{\mathrm{mix}}$ & $\mathcal{D}_{\mathrm{bal}}$ & $\mathrm{CondErr}_{Y\mid AB}$ & AUC\\
\hline
0.05 & OneHot-Diffusion & $0.016\pm0.003$ & $0.398\pm0.009$ & $0.408\pm0.013$ & $0.0124\pm0.0019$ & $0.608\pm0.004$\\
0.05 & Logit-Diffusion & $0.016\pm0.003$ & $0.392\pm0.008$ & $0.400\pm0.010$ & $0.0122\pm0.0035$ & $0.606\pm0.003$\\
0.01 & OneHot-Diffusion & $0.017\pm0.003$ & $0.390\pm0.008$ & $0.399\pm0.013$ & $0.0143\pm0.0031$ & $0.606\pm0.004$\\
0.01 & Logit-Diffusion & $0.016\pm0.003$ & $0.383\pm0.007$ & $0.392\pm0.009$ & $0.0143\pm0.0036$ & $0.603\pm0.004$\\
\hline
\end{tabular}}
\end{table}

\begin{table}[ht]
\centering
\small
\caption{Gaussian diffusion rare-cell diagnostics for $g_{\mathrm{rare}}=(A=3,B=2,Y=1)$, whose true mass is $\rho$. The columns report generated mass, absolute mass error, and conditional numerical $W_2$ within $g_{\mathrm{rare}}$. Entries are mean $\pm$ SD over ten seeds; lower error is better. The representations have similar rare-cell accuracy: Logit-Diffusion has slightly smaller mass error at both $\rho$ values, while the $W_2$ ordering changes with $\rho$.}
\label{tab:sim_diffusion_rare}
\begin{tabular}{llccc}
\hline
$\rho$ & Method & $\widehat\pi_{g_{\mathrm{rare}}}$ & $|\widehat\pi_{g_{\mathrm{rare}}}-\pi_{g_{\mathrm{rare}}}^\star|$ & Rare-cell $W_2$\\
\hline
0.05 & OneHot-Diffusion & $0.049\pm0.002$ & $0.0019\pm0.0010$ & $0.352\pm0.045$\\
0.05 & Logit-Diffusion & $0.050\pm0.002$ & $0.0013\pm0.0010$ & $0.345\pm0.044$\\
0.01 & OneHot-Diffusion & $0.010\pm0.001$ & $0.0009\pm0.0005$ & $0.407\pm0.077$\\
0.01 & Logit-Diffusion & $0.009\pm0.001$ & $0.0008\pm0.0008$ & $0.415\pm0.090$\\
\hline
\end{tabular}
\end{table}

The diffusion results provide a second within-family representation comparison. Under the same diffusion architecture and training procedure, replacing one-hot categorical coordinates with scaled-logit coordinates gives slightly lower reported mean mixed-distribution metrics in most runs. The difference is modest but consistent across the two imbalance regimes, suggesting that the same coordinate choice can matter within diffusion as well. We do not interpret these results as a direct algorithmic comparison between diffusion and Flow Matching; each model family is used to test the effect of the categorical coordinate system under a fixed generative mechanism.

\section{Real-Data Experiments}
\label{sec:real_data}

We evaluate the logit-coordinate framework on four real mixed-type tabular datasets: Adult, Churn2, Cardio, and Buddy. These datasets contain multiple categorical variables, non-Gaussian numerical variables, label-dependent structure, and higher-order dependence across feature blocks. The experiments address two main questions. First, within a fixed continuous generative model family, does replacing one-hot categorical coordinates with scaled-logit coordinates improve mixed-distribution fidelity? Second, does an explicit block-conditional factorization further improve the preservation of label--categorical--numerical dependence?

To reduce sensitivity to a particular train--validation--test partition, each method is evaluated over ten data splits. These consist of the original benchmark split and nine additional random splits generated using distinct seeds. For a given split, all methods use the same training, validation, and test observations. We report the mean and standard deviation of each evaluation metric over the ten runs. The same model variants and probabilistic factorizations are used on all four datasets; in particular, no dataset-specific Logit-FM variant is introduced for Churn2.

\subsection{Experimental Setup}

For the continuous-representation models, categorical variables are embedded into Euclidean coordinates before applying either Flow Matching or Gaussian diffusion. We write $X_{\mathrm{cat}}$ and $X_{\mathrm{num}}$ for the original categorical and numerical variables, and $Z_{\mathrm{cat}}$ and $Z_{\mathrm{num}}$ for their transformed continuous coordinates. Specifically,
\[
Z_{\mathrm{cat}}=E(X_{\mathrm{cat}})
\]
is either a scaled-logit or one-hot categorical embedding, while
\[
Z_{\mathrm{num}}=T(X_{\mathrm{num}})
\]
is the numerical transformation used by the corresponding generative model.
 The Flow Matching models use empirical Gaussian-quantile transformations for
numerical variables, whereas the Gaussian diffusion models use standardized
numerical coordinates.  Probabilistic factorizations are expressed in the
original variables, whereas Flow Matching and diffusion dynamics are
implemented in the transformed $Z$-coordinates.  For every
continuous-representation model, generated categorical coordinates are
decoded by maximum probability; this holds for both the one-hot and
scaled-logit variants.

\paragraph{Flow Matching models.}
We compare three Flow Matching models. One-Hot FM and Logit FM are flat joint models that transform each mixed-type observation into a single continuous vector,
\[
Z=(Z_{\mathrm{num}},Z_{\mathrm{cat}},Z_y),
\]
and train one joint Flow Matching velocity field. The two models use the same joint-generation framework and differ in their categorical representations: One-Hot FM uses standard one-hot coordinates, whereas Logit FM uses scaled-logit coordinates.
 The flat Flow Matching network has hidden dimension 512 and six residual
blocks, is trained with learning rate $10^{-3}$, and is sampled by a 200-step
Euler ODE solver. These settings are shared by the one-hot and logit variants.

The Block-Conditional Logit FM model uses the factorization
\[
P(y,X_{\mathrm{cat}},X_{\mathrm{num}})
=
P(y)
P(X_{\mathrm{cat}}\mid y)
P(X_{\mathrm{num}}\mid X_{\mathrm{cat}},y).
\]
It first samples the label, then generates the categorical block conditional on the label, and finally generates the numerical block conditional on both the generated categorical variables and the label. This model uses the same scaled-logit representation as Logit FM but introduces additional conditional structure in the joint generator. The definitions of all three Flow Matching models are held fixed across Adult, Churn2, Cardio, and Buddy.

\paragraph{Continuous-representation diffusion models.}
We also evaluate One-Hot Diffusion and Logit Diffusion. Both use Gaussian diffusion in transformed continuous coordinates and the same block-conditional factorization,
\[
P(y,X_{\mathrm{cat}},X_{\mathrm{num}})
=
P(y)
P(X_{\mathrm{cat}}\mid y)
P(X_{\mathrm{num}}\mid X_{\mathrm{cat}},y).
\]
The two methods use the same diffusion parameterization and training protocol, differing only in whether categorical variables are represented by one-hot or scaled-logit coordinates. Neither method uses multinomial categorical diffusion.

\paragraph{TabDDPM baselines.}
TabDDPM \citep{kotelnikov2023tabddpm} is included as an external mixed-type diffusion baseline. We report two baseline versions, denoted TabDDPM-MLP and TabDDPM-CB, corresponding to the MLP and CatBoost versions used in our TabDDPM experiments. Both variants are evaluated on the same ten data splits as the continuous-representation models. Reporting the two versions separately avoids combining results obtained from different TabDDPM configurations.

  The seven evaluated methods are  One-Hot FM, Logit FM, Block-Conditional Logit
FM, One-Hot Diffusion, Logit Diffusion,   and the MLP and CB variants of TabDDPM.

\subsection{Evaluation Metrics}
\label{subsec:real_metrics}

We evaluate distributional fidelity, categorical dependence, and downstream predictive utility. The primary   criterion  is the standardized mixed discrepancy
\[
\mathcal{D}_{\mathrm{mix}}
=
TV(\pi_y^\star,\widehat{\pi}_y)
+
\sum_{c\in\mathcal{Y}}
\pi_y^\star(c)
W_2\!\left(
P^\star(X_{\mathrm{num}}\mid y=c),
\widehat P(\widehat X_{\mathrm{num}}\mid \widehat y=c)
\right),
\]
where numerical variables are standardized before computing the Wasserstein term. We also report the weighted conditional $W_2$ component separately.

Categorical fidelity is assessed using exact total variation over the joint categorical-label configurations and mean marginal total variation across the categorical predictors. Pairwise categorical dependence is measured by the mean absolute error in Cramér's $V$. Downstream utility is evaluated using train-on-synthetic, test-on-real macro-F1 (TSTR).

Each metric is computed for the original benchmark split and nine additional random splits. Results are reported as mean $\pm$ standard deviation over the ten runs. Lower values are better for the distributional and dependence metrics, whereas higher values are better for TSTR.

\subsection{Results}
\label{subsec:real_results}

Tables~\ref{tab:adult_distribution}--\ref{tab:buddy_diagnostics} report the ten-run results. The primary comparisons are within model family: One-Hot FM versus Logit FM, and One-Hot Diffusion versus Logit Diffusion. TabDDPM-MLP and TabDDPM-CB are reported separately as external mixed-type diffusion baselines.

\paragraph{Adult.}
Within Flow Matching, Logit FM improves all distributional and dependence metrics over One-Hot FM, with similar TSTR. Block-Conditional Logit FM further improves every reported metric relative to flat Logit FM. Within diffusion, Logit Diffusion has lower $\mathcal{D}_{\mathrm{mix}}$, conditional $W_2$, and Cramér's $V$ error, whereas One-Hot Diffusion has slightly lower exact and marginal categorical TV. The two diffusion representations have nearly identical TSTR. TabDDPM-CB is more competitive than TabDDPM-MLP on the primary distributional metrics, but also exhibits substantially greater variability across runs.

\begin{table}[ht]
\centering
\caption{Adult distributional fidelity (mean $\pm$ SD over the benchmark split and nine random splits). $\mathcal{D}_{\mathrm{mix}}$ is label TV plus label-weighted conditional numerical $W_2$ on standardized features; Conditional $W_2$ is its transport term; exact categorical TV compares joint categorical--label masses. Lower is better; boldface marks the best mean. The three column minima are attained by Logit Diffusion, TabDDPM-CB, and One-Hot Diffusion, respectively.}
\label{tab:adult_distribution}
\begin{tabular}{lccc}
\toprule
Method
& $\mathcal{D}_{\mathrm{mix}}$
& Conditional $W_2$
& Exact categorical TV \\
\midrule
One-Hot FM
& $0.4967 \pm 0.0186$
& $0.3708 \pm 0.0170$
& $0.4966 \pm 0.0073$ \\

Logit FM
& $0.4128 \pm 0.0141$
& $0.3421 \pm 0.0145$
& $0.3928 \pm 0.0068$ \\

Block-Cond.\ Logit FM
& $0.3798 \pm 0.0125$
& $0.3245 \pm 0.0111$
& $0.3427 \pm 0.0032$ \\

One-Hot Diffusion
& $0.3500 \pm 0.0134$
& $0.3309 \pm 0.0132$
& $\mathbf{0.2578 \pm 0.0018}$ \\

Logit Diffusion
& $\mathbf{0.3452 \pm 0.0111}$
& $0.3256 \pm 0.0104$
& $0.2630 \pm 0.0025$ \\

TabDDPM-MLP
& $0.4995 \pm 0.0548$
& $0.4695 \pm 0.0555$
& $0.2809 \pm 0.0028$ \\

TabDDPM-CB
& $0.3576 \pm 0.0688$
& $\mathbf{0.3238 \pm 0.0685}$
& $0.2729 \pm 0.0047$ \\
\bottomrule
\end{tabular}
\end{table}

\begin{table}[ht]
\centering
 \setlength{\tabcolsep}{5pt}
 \caption{Adult categorical fidelity and utility (mean $\pm$ SD over ten splits). Marginal TV averages categorical-predictor TV; Cramér's $V$ error averages pairwise-dependence error; TSTR is train-on-synthetic, test-on-real macro-F1. Lower errors and higher TSTR are better; boldface marks the best mean. Block-Conditional Logit FM minimizes marginal TV, and Logit Diffusion is best on the other two metrics.}
\label{tab:adult_diagnostics}
\begin{tabular}{lccc}
\toprule
Method
& Marginal categorical TV
& Cramér's $V$ error
& TSTR \\
\midrule
One-Hot FM
& $0.0307 \pm 0.0023$
& $0.0526 \pm 0.0019$
& $0.7968 \pm 0.0027$ \\

Logit FM
& $0.0246 \pm 0.0035$
& $0.0368 \pm 0.0026$
& $0.7956 \pm 0.0030$ \\

Block-Cond.\ Logit FM
& $\mathbf{0.0052 \pm 0.0005}$
& $0.0168 \pm 0.0012$
& $0.8012 \pm 0.0025$ \\

One-Hot Diffusion
& $0.0073 \pm 0.0005$
& $0.0062 \pm 0.0010$
& $0.8079 \pm 0.0031$ \\

Logit Diffusion
& $0.0080 \pm 0.0007$
& $\mathbf{0.0058 \pm 0.0009}$
& $\mathbf{0.8081 \pm 0.0027}$ \\

TabDDPM-MLP
& $0.0154 \pm 0.0006$
& $0.0080 \pm 0.0040$
& $0.7977 \pm 0.0022$ \\

TabDDPM-CB
& $0.0173 \pm 0.0009$
& $0.0092 \pm 0.0006$
& $0.8006 \pm 0.0027$ \\
\bottomrule
\end{tabular}
\end{table}

\paragraph{Churn2.}
Flat Logit FM and One-Hot FM perform similarly, with no consistent advantage in the primary distributional metrics. Block-Conditional Logit FM substantially improves $\mathcal{D}_{\mathrm{mix}}$, conditional $W_2$, categorical fidelity, and TSTR relative to both flat FM models. Within diffusion, Logit Diffusion consistently outperforms One-Hot Diffusion and achieves the best mean in all six reported metrics. TabDDPM-CB improves conditional $W_2$ and TSTR over TabDDPM-MLP, but has markedly worse categorical fidelity.

\begin{table}[htb]
\centering
\caption{Churn2 distributional fidelity (mean $\pm$ SD over the benchmark split and nine random splits). $\mathcal{D}_{\mathrm{mix}}$ is label TV plus label-weighted conditional numerical $W_2$ on standardized features; Conditional $W_2$ is its transport term; exact categorical TV compares joint categorical--label masses. Lower is better; boldface marks the best mean. Logit Diffusion is best overall, and Block-Conditional Logit FM is best among the Flow Matching models, on all three metrics.}
\label{tab:churn2_distribution}
\begin{tabular}{lccc}
\toprule
Method
& $\mathcal{D}_{\mathrm{mix}}$
& Conditional $W_2$
& Exact categorical TV \\
\midrule
One-Hot FM
& $0.2973 \pm 0.0047$
& $0.2762 \pm 0.0050$
& $0.0211 \pm 0.0009$ \\

Logit FM
& $0.3009 \pm 0.0107$
& $0.2797 \pm 0.0094$
& $0.0212 \pm 0.0025$ \\

Block-Cond.\ Logit FM
& $0.2637 \pm 0.0057$
& $0.2445 \pm 0.0057$
& $0.0193 \pm 0.0011$ \\

One-Hot Diffusion
& $0.1499 \pm 0.0041$
& $0.1364 \pm 0.0041$
& $0.0135 \pm 0.0019$ \\

Logit Diffusion
& $\mathbf{0.1436 \pm 0.0036}$
& $\mathbf{0.1314 \pm 0.0029}$
& $\mathbf{0.0122 \pm 0.0017}$ \\

TabDDPM-MLP
& $0.3115 \pm 0.0091$
& $0.2900 \pm 0.0088$
& $0.0216 \pm 0.0018$ \\

TabDDPM-CB
& $0.2835 \pm 0.0079$
& $0.2232 \pm 0.0064$
& $0.0603 \pm 0.0029$ \\
\bottomrule
\end{tabular}
\end{table}

\begin{table}[htb]
\centering
 \setlength{\tabcolsep}{5pt}
 \caption{Churn2 categorical fidelity and utility (mean $\pm$ SD over ten splits). Marginal TV averages categorical-predictor TV; Cramér's $V$ error averages pairwise-dependence error; TSTR is train-on-synthetic, test-on-real macro-F1. Lower errors and higher TSTR are better; boldface marks the best mean. Block-Conditional Logit FM minimizes marginal TV, and Logit Diffusion is best on the other two metrics.}
\label{tab:churn2_diagnostics}
\begin{tabular}{lccc}
\toprule
Method
& Marginal categorical TV
& Cramér's $V$ error
& TSTR \\
\midrule
One-Hot FM
& $0.0042 \pm 0.0018$
& $0.0066 \pm 0.0016$
& $0.7637 \pm 0.0061$ \\

Logit FM
& $0.0035 \pm 0.0011$
& $0.0064 \pm 0.0009$
& $0.7624 \pm 0.0060$ \\

Block-Cond.\ Logit FM
& $\mathbf{0.0017 \pm 0.0006}$
& $0.0055 \pm 0.0007$
& $0.7752 \pm 0.0041$ \\

One-Hot Diffusion
& $0.0031 \pm 0.0010$
& $0.0037 \pm 0.0013$
& $0.8205 \pm 0.0045$ \\

Logit Diffusion
& $0.0022 \pm 0.0009$
& $\mathbf{0.0036 \pm 0.0008}$
& $\mathbf{0.8224 \pm 0.0040}$ \\

TabDDPM-MLP
& $0.0030 \pm 0.0013$
& $0.0061 \pm 0.0014$
& $0.7703 \pm 0.0055$ \\

TabDDPM-CB
& $0.0215 \pm 0.0016$
& $0.0177 \pm 0.0013$
& $0.7937 \pm 0.0047$ \\
\bottomrule
\end{tabular}
\end{table}

\paragraph{Cardio.}
Logit FM improves the primary distributional metrics over One-Hot FM, and Block-Conditional Logit FM provides a further substantial reduction in $\mathcal{D}_{\mathrm{mix}}$ and conditional $W_2$. Within diffusion, Logit Diffusion consistently improves over One-Hot Diffusion and achieves the best mean for all three distributional metrics, Cramér's $V$ error, and TSTR. One-Hot FM and Block-Conditional Logit FM tie for the lowest marginal categorical TV at the reported precision. TabDDPM-CB improves $\mathcal{D}_{\mathrm{mix}}$, conditional $W_2$, and TSTR over TabDDPM-MLP, but performs worse on categorical fidelity and dependence.

\begin{table}[htb]
\centering
\caption{Cardio distributional fidelity (mean $\pm$ SD over the benchmark split and nine random splits). $\mathcal{D}_{\mathrm{mix}}$ is label TV plus label-weighted conditional numerical $W_2$ on standardized features; Conditional $W_2$ is its transport term; exact categorical TV compares joint categorical--label masses. Lower is better; boldface marks the best mean. Logit Diffusion is best overall on all three metrics; within Flow Matching, Logit FM improves on One-Hot FM and block conditioning improves further.}
\label{tab:cardio_distribution}
\begin{tabular}{lccc}
\toprule
Method
& $\mathcal{D}_{\mathrm{mix}}$
& Conditional $W_2$
& Exact categorical TV \\
\midrule
One-Hot FM
& $0.8383 \pm 0.0892$
& $0.8257 \pm 0.0890$
& $0.0147 \pm 0.0009$ \\

Logit FM
& $0.7575 \pm 0.0872$
& $0.7462 \pm 0.0879$
& $0.0133 \pm 0.0011$ \\

Block-Cond.\ Logit FM
& $0.6013 \pm 0.0798$
& $0.5917 \pm 0.0796$
& $0.0120 \pm 0.0004$ \\

One-Hot Diffusion
& $0.4060 \pm 0.0620$
& $0.3963 \pm 0.0613$
& $0.0121 \pm 0.0010$ \\

Logit Diffusion
& $\mathbf{0.3814 \pm 0.0505}$
& $\mathbf{0.3729 \pm 0.0506}$
& $\mathbf{0.0106 \pm 0.0007}$ \\

TabDDPM-MLP
& $0.9592 \pm 0.0873$
& $0.9460 \pm 0.0876$
& $0.0143 \pm 0.0006$ \\

TabDDPM-CB
& $0.7112 \pm 0.0834$
& $0.6781 \pm 0.0836$
& $0.0332 \pm 0.0013$ \\
\bottomrule
\end{tabular}
\end{table}

\begin{table}[htb]
\centering
 \setlength{\tabcolsep}{5pt}
 \caption{Cardio categorical fidelity and utility (mean $\pm$ SD over ten splits). Marginal TV averages categorical-predictor TV; Cramér's $V$ error averages pairwise-dependence error; TSTR is train-on-synthetic, test-on-real macro-F1. Lower errors and higher TSTR are better; boldface marks the best mean at the reported precision. One-Hot FM and Block-Conditional Logit FM tie on marginal TV, and Logit Diffusion is best on the other two metrics.}
\label{tab:cardio_diagnostics}
\begin{tabular}{lccc}
\toprule
Method
& Marginal categorical TV
& Cramér's $V$ error
& TSTR \\
\midrule
One-Hot FM
& $\mathbf{0.0010 \pm 0.0001}$
& $0.0048 \pm 0.0005$
& $0.7366 \pm 0.0014$ \\

Logit FM
& $0.0012 \pm 0.0004$
& $0.0040 \pm 0.0006$
& $0.7371 \pm 0.0013$ \\

Block-Cond.\ Logit FM
& $\mathbf{0.0010 \pm 0.0002}$
& $0.0028 \pm 0.0003$
& $0.7386 \pm 0.0012$ \\

One-Hot Diffusion
& $0.0014 \pm 0.0002$
& $0.0028 \pm 0.0005$
& $0.7398 \pm 0.0011$ \\

Logit Diffusion
& $0.0011 \pm 0.0004$
& $\mathbf{0.0027 \pm 0.0004}$
& $\mathbf{0.7403 \pm 0.0012}$ \\

TabDDPM-MLP
& $0.0032 \pm 0.0003$
& $0.0044 \pm 0.0006$
& $0.7360 \pm 0.0014$ \\

TabDDPM-CB
& $0.0092 \pm 0.0006$
& $0.0120 \pm 0.0007$
& $0.7385 \pm 0.0016$ \\
\bottomrule
\end{tabular}
\end{table}

\paragraph{Buddy.}
Logit FM improves substantially over One-Hot FM across all reported metrics. Block-Conditional Logit FM provides a further large reduction in $\mathcal{D}_{\mathrm{mix}}$ and conditional $W_2$, and gives the best Flow Matching results overall. Within diffusion, One-Hot and Logit Diffusion are very close: One-Hot Diffusion has slightly better exact categorical TV, Cramér's $V$ error, and TSTR, while Logit Diffusion has slightly lower $\mathcal{D}_{\mathrm{mix}}$ and conditional $W_2$. TabDDPM-CB consistently improves over TabDDPM-MLP but remains behind the block-conditional FM and diffusion models on the primary fidelity metrics.

\begin{table}[htb]
\centering
\caption{Buddy distributional fidelity (mean $\pm$ SD over the benchmark split and nine random splits). $\mathcal{D}_{\mathrm{mix}}$ is label TV plus label-weighted conditional numerical $W_2$ on standardized features; Conditional $W_2$ is its transport term; exact categorical TV compares joint categorical--label masses. Lower is better; boldface marks the best mean. Block-Conditional Logit FM minimizes the first two metrics, and One-Hot Diffusion minimizes exact categorical TV.}
\label{tab:buddy_distribution}
\begin{tabular}{lccc}
\toprule
Method
& $\mathcal{D}_{\mathrm{mix}}$
& Conditional $W_2$
& Exact categorical TV \\
\midrule
One-Hot FM
& $0.4453 \pm 0.0137$
& $0.2663 \pm 0.0081$
& $0.2943 \pm 0.0076$ \\

Logit FM
& $0.3177 \pm 0.0157$
& $0.2532 \pm 0.0082$
& $0.1507 \pm 0.0104$ \\

Block-Cond.\ Logit FM
& $\mathbf{0.1826 \pm 0.0057}$
& $\mathbf{0.1409 \pm 0.0043}$
& $0.1096 \pm 0.0015$ \\

One-Hot Diffusion
& $0.2032 \pm 0.0080$
& $0.1778 \pm 0.0082$
& $\mathbf{0.0761 \pm 0.0027}$ \\

Logit Diffusion
& $0.2017 \pm 0.0074$
& $0.1751 \pm 0.0071$
& $0.0801 \pm 0.0023$ \\

TabDDPM-MLP
& $0.2776 \pm 0.0100$
& $0.2140 \pm 0.0089$
& $0.0908 \pm 0.0019$ \\

TabDDPM-CB
& $0.2232 \pm 0.0051$
& $0.1653 \pm 0.0056$
& $0.0773 \pm 0.0014$ \\
\bottomrule
\end{tabular}
\end{table}

\begin{table}[htb]
\centering
 \setlength{\tabcolsep}{5pt}
 \caption{Buddy categorical fidelity and utility (mean $\pm$ SD over ten splits). Marginal TV averages categorical-predictor TV; Cramér's $V$ error averages pairwise-dependence error; TSTR is train-on-synthetic, test-on-real macro-F1. Lower errors and higher TSTR are better; boldface marks the best mean. Block-Conditional Logit FM minimizes marginal TV, and One-Hot Diffusion is best on the other two metrics.}
\label{tab:buddy_diagnostics}
\begin{tabular}{lccc}
\toprule
Method
& Marginal categorical TV
& Cramér's $V$ error
& TSTR \\
\midrule
One-Hot FM
& $0.0318 \pm 0.0019$
& $0.0756 \pm 0.0030$
& $0.9078 \pm 0.0023$ \\

Logit FM
& $0.0168 \pm 0.0035$
& $0.0328 \pm 0.0035$
& $0.9114 \pm 0.0031$ \\

Block-Cond.\ Logit FM
& $\mathbf{0.0064 \pm 0.0010}$
& $0.0177 \pm 0.0019$
& $0.9356 \pm 0.0016$ \\

One-Hot Diffusion
& $0.0090 \pm 0.0013$
& $\mathbf{0.0082 \pm 0.0039}$
& $\mathbf{0.9623 \pm 0.0007}$ \\

Logit Diffusion
& $0.0092 \pm 0.0007$
& $0.0121 \pm 0.0035$
& $0.9618 \pm 0.0011$ \\

TabDDPM-MLP
& $0.0170 \pm 0.0016$
& $0.0188 \pm 0.0029$
& $0.9247 \pm 0.0020$ \\

TabDDPM-CB
& $0.0156 \pm 0.0012$
& $0.0141 \pm 0.0017$
& $0.9300 \pm 0.0021$ \\
\bottomrule
\end{tabular}
\end{table}

\subsection{Summary of Findings}

The ten-run experiments show three main patterns. First, Logit FM improves the primary distributional metrics over One-Hot FM on Adult, Cardio, and Buddy, while the two flat representations are comparable on Churn2. Second, Block-Conditional Logit FM consistently improves over flat Logit FM across all four datasets, indicating that categorical representation and conditional factorization address complementary aspects of mixed-data fidelity.

Third, Logit Diffusion improves over One-Hot Diffusion on most metrics for Adult, Churn2, and Cardio. The two diffusion representations are nearly tied on Buddy, with different metrics favoring each representation. The separate TabDDPM results also show that its MLP and CatBoost variants can behave differently: the CatBoost version generally improves conditional Wasserstein error and TSTR, but may worsen categorical fidelity.

Overall, the results provide uncertainty-quantified, within-family evidence that logit coordinates can improve mixed tabular generation, while explicit conditional structure is important for preserving higher-order dependence. Comparisons across Flow Matching, diffusion, and TabDDPM should remain descriptive because these model families differ in architecture, objective, and training procedure.

\section{Discussion}
The theoretical and empirical results support a representation-level view of mixed tabular generation. Categorical variables are not ordinary Euclidean coordinates: their distributions live on probability simplices, and their statistical behavior is especially delicate under imbalance. The proposed logit-coordinate framework addresses this issue by embedding categorical variables in smoothed natural-parameter coordinates before applying continuous generative dynamics. This construction applies across generative algorithms. Flow Matching and diffusion use different training objectives and sampling procedures, and both can operate on the same transformed logit representation.

\medskip
\noindent\textit{Choice of fidelity discrepancy.}
The mixed discrepancy $\mathcal{D}_{\mathrm{mix}}$ is designed to respect the intrinsic mixture structure of mixed continuous--categorical distributions. Errors in generative modeling arise in two distinct forms: misallocation of categorical mass and distortion of class-conditional continuous geometry. By separating these components, $\mathcal{D}_{\mathrm{mix}}$ isolates marginal imbalance from within-class transport distortion and makes explicit which aspect of the generator contributes to overall error.

Wasserstein distances provide a natural notion of distributional comparison for transport-based generative models \citep{Arjovsky2017, villani2009optimal, Peyre2019}. They metrize weak convergence under suitable moment conditions and are stable under Lipschitz pushforward maps, properties that are well suited for analyzing continuous generative dynamics. In this setting, the Wasserstein component governs the stability of class-conditional continuous distributions, while categorical marginal errors are controlled through the stability of logit-coordinate decoding.

While a joint optimal transport metric could be imposed on the product space $\mathbb R^p\times\{1,\ldots,K\}$, as formalized in Appendix~A.3.3, such constructions require a scale parameter to balance continuous distortion against categorical mismatch. The proposed decomposition avoids this calibration choice and aligns with the separate mechanisms governing categorical mass and conditional continuous geometry in logit-coordinate generative models.

\medskip
\noindent\textit{Stability and transport geometry.}
The stability analysis helps explain why the logit representation can improve reported mixed-data fidelity in the evaluated model families. After categorical variables are embedded in logit coordinates, both Flow Matching and diffusion define continuous generative dynamics in a Euclidean transformed space. Errors in the learned dynamics then propagate to the decoded mixed distribution through two steps: Wasserstein stability in transformed coordinates and softmax stability of the categorical decoding map. This perspective separates the geometric effect of the categorical representation from the particular training objective used to learn the dynamics. Thus, empirical gains of logit-coordinate models are best interpreted as representation-level improvements.

\medskip
\noindent\textit{Rare-cell imbalance.}
The convergence analysis links the representation question to the statistical difficulty of mixed-data generation. For a categorical cell with population mass $\pi_k^\star$, the conditional component is learned from an effective sample size of order $n\pi_k^\star$. Rare cells therefore create two coupled demands: a generator must preserve their probability mass and estimate their conditional numerical laws from limited observations. This mechanism is consistent with the controlled simulations, in which the benefit of scaled-logit coordinates is most visible under severe imbalance. It also clarifies why categorical representation and conditional modeling must be considered together.

\medskip
\noindent\textit{Representation versus conditional structure.}
Mixed tabular generation therefore involves two related but distinct design problems. The first is geometric: categorical variables should be represented in coordinates that respect their probabilistic structure. The scaled-logit embedding addresses this issue by mapping categorical observations into smoothed natural-parameter coordinates before applying continuous generative dynamics. The second is structural: realistic tabular data often contain strong dependence across labels, categorical variables, and numerical variables. A flat joint model may reproduce marginal distributions reasonably well while still failing to preserve conditional numerical structure within categorical-label cells. The block-conditional extension addresses this issue by explicitly factorizing the joint distribution into label, categorical, and numerical components.

\medskip
\noindent\textit{Interpreting Flow Matching and diffusion comparisons.}
The Flow Matching and diffusion experiments provide complementary tests of the same representation principle rather than a direct ranking of generative algorithms. Because the two families differ in their objectives, architectures, conditioning mechanisms, and sampling procedures, the primary comparisons are within family: One-Hot FM versus Logit FM, and One-Hot Diffusion versus Logit Diffusion. These comparisons isolate the effect of replacing one-hot coordinates with scaled-logit coordinates while holding the model family fixed. Results over the original benchmark partition and nine additional random partitions show that Logit FM improves the primary distributional metrics on three datasets and is comparable on Churn2, while Logit Diffusion generally improves over or matches One-Hot Diffusion. The two TabDDPM variants serve as external references rather than a basis for direct cross-family ranking.

\medskip
\noindent\textit{Limitations and future directions.}
The present results leave several questions open. First, logit coordinates address categorical representation, whereas higher-order dependence in mixed tabular data requires additional structural modeling. The block-conditional results indicate that preserving numerical structure within categorical-label cells may require additional architectural structure.

Second, the diffusion experiments evaluate only particular Gaussian diffusion parameterizations in transformed coordinates. Performance may depend on the denoising objective, noise schedule, conditioning mechanism, and decoding rule.

Third, the theory isolates the representation mechanism under population and nonparametric estimation conditions. A guarantee for the implemented neural procedures would also need to account for optimization error, finite-step sampling error, hyperparameter sensitivity, and model-selection effects. The experiments over ten data splits quantify variability across partitions and runs, while broader studies should examine sensitivity to computational budgets and design choices such as smoothing, logit scaling, the ALR reference category, and categorical decoding.

Taken together, the analysis and experiments identify the categorical coordinate system as a substantive design choice for mixed tabular generators. Logit coordinates provide a common Euclidean representation for Flow Matching and diffusion, while block-conditional modeling supplies additional structure for preserving conditional dependence. Promising directions include richer conditional architectures, support-aware diffusion mechanisms, Fisher-weighted objectives, alternative log-ratio representations, and finite-sample theory that incorporates optimization and numerical sampling error.

\section*{Acknowledgments}
This work was supported in part by NSF grant DMS-1952539 and NIH grants R01GM113250, R01GM126002, R01AG065636, R01AG074858, R01AG069895, and U01AG073079.

\newpage
\appendix
\section{Supplemental Proofs}
\label{app:proofs}

\subsection{Proof Details for Stability Results}
\label{app:stability-proofs}

\medskip
\noindent\textit{Proof of Theorem~\ref{thm:unified_stability} and Theorem~\ref{thm:fm_w2_stability}.}
Couple the two terminal laws by using the same initial random variable $X_0\sim P_0$. Let $X_t^\star=\Phi_t^{f^\star}(X_0)$ and $\widehat X_t=\Phi_t^{\hat f}(X_0)$. With $\Delta_t=\widehat X_t-X_t^\star$,
\[
\dot\Delta_t
=
\hat f(t,\widehat X_t)-f^\star(t,X_t^\star)
=
\{\hat f(t,\widehat X_t)-\hat f(t,X_t^\star)\}
+
\{\hat f(t,X_t^\star)-f^\star(t,X_t^\star)\}.
\]
The Lipschitz condition gives, for almost every $t$,
\[
\frac{d}{dt}\|\Delta_t\|
\le
L\|\Delta_t\|+
\|\hat f(t,X_t^\star)-f^\star(t,X_t^\star)\|.
\]
Since $\Delta_0=0$, Gronwall's inequality yields
\[
\|\Delta_1\|
\le
e^L\int_0^1\|\hat f(t,X_t^\star)-f^\star(t,X_t^\star)\|\,dt.
\]
Taking $L^2$ norms and applying Minkowski's inequality gives the stated bound. The Flow Matching statement is the same argument with $f=v$.

\medskip
\noindent\textit{Proof of Corollary~\ref{cor:tv_from_logit_fm}.}
Use the synchronous coupling $Z_t^\star=\Phi_t^{v_Z^\star}(Z_0)$ and $\widehat Z_t=\Phi_t^{\hat v_Z}(Z_0)$. With $\Pi=S(Z_1^\star)$ and $\widehat\Pi=S(\widehat Z_1)$,
  \[
TV(\pi^\star,\widehat\pi)
\le
\mathrm{TV}(\pi^\star,\pi^{\star,(\varepsilon)})
+
\frac{1}{2}\|\mathbb E\Pi-\mathbb E\widehat\Pi\|_1
\le
\varepsilon
+
\frac{1}{2}\mathbb E\|S(Z_1^\star)-S(\widehat Z_1)\|_1.
\] 
The Lipschitz property of $S$ and the preceding flow-stability bound applied in logit space give \eqref{eq:tv_logit_fm_bound}.

\medskip
\noindent\textit{Proof of Corollary~\ref{cor:mixed_fidelity_fm}.}
The categorical part follows from Corollary~\ref{cor:tv_from_logit_fm}. For each class $k$, Theorem~\ref{thm:fm_w2_stability} gives the conditional continuous bound
\[
W_2(P_k^{(c),\star},\widehat P_k^{(c)})
\le
e^{L_k}\int_0^1
\left(\mathbb E\|\hat v_k(t,X_{t,k}^\star)-v_k^\star(t,X_{t,k}^\star)\|_2^2\right)^{1/2}dt.
\]
Multiplying by $\pi_k^\star$, summing over $k$, and adding the total-variation term gives \eqref{eq:mixed_fidelity_fm_bound}.

\medskip
\noindent\textit{Proof of Lemma~\ref{lem:joint_transport_upper}.}
Take a maximal coupling $(D,\widehat D)$ of the categorical marginals, so that $\mathbb P(D\ne\widehat D)=TV(\pi^\star,\widehat\pi)$. Conditional on $D=\widehat D=k$, couple the continuous components by an optimal $W_2$ coupling of $P_k^{(c),\star}$ and $\widehat P_k^{(c)}$. Conditional on $D\ne\widehat D$, use any coupling and the conditional second-moment bound from Assumption~\ref{ass:second_moment} together with
$\|x-y\|_2^2\le 2\|x\|_2^2+2\|y\|_2^2$.
The matched part contributes at most $\sum_k\pi_k^\star W_2^2(P_k^{(c),\star},\widehat P_k^{(c)})$, while the mismatched part contributes at most $4M TV(\pi^\star,\widehat\pi)$ plus the categorical cost $\lambda TV(\pi^\star,\widehat\pi)$.

\medskip
\noindent\textit{Proof of Theorem~\ref{thm:joint_ot_stability}.}
Substitute the class-conditional Wasserstein bounds from Theorem~\ref{thm:fm_w2_stability} and the categorical total-variation bound from Corollary~\ref{cor:tv_from_logit_fm} into Lemma~\ref{lem:joint_transport_upper}.

\medskip
\noindent\textit{Proof of Proposition~\ref{prop:diffusion_stability} and Corollary~\ref{cor:mixed_diffusion_stability}.}
The probability flow ODE has drift $f^\star(t,x)=\mu(t,x)-\frac{1}{2}\sigma^2(t)s^\star(t,x)$, and the learned drift replaces $s^\star$ by $\widehat s$. Applying Theorem~\ref{thm:unified_stability} to these two drifts gives a bound in terms of $\|\widehat f-f^\star\|$. The identity
\[
\widehat f(t,x)-f^\star(t,x)=-\frac{1}{2}\sigma^2(t)\{\widehat s(t,x)-s^\star(t,x)\}
\]
converts the drift error to score-estimation error. Applying the same argument componentwise to the logit and class-conditional probability-flow ODEs gives the mixed diffusion bound.
\subsection{Spectral Bound for the Softmax Jacobian}

\begin{lemma}[Spectral bound for the categorical Fisher matrix]
\label{lem:softmax-spectral-bound}
For any probability vector $p\in\Delta_K$,
\[
\|\mathrm{diag}(p)-pp^\top\|_{\mathrm{op}}
\le
\frac{1}{2}.
\]
Moreover, the bound is sharp and is attained when
$p=(1/2,1/2,0,\dots,0)$.
\end{lemma}

\begin{proof}
Define
\[
M(p)=\mathrm{diag}(p)-pp^\top.
\]
This matrix is symmetric and positive semidefinite.
It equals the covariance matrix of the one-hot random vector
$e_D$, where $P(D=k)=p_k$.

For any unit vector $u\in\mathbb{R}^K$,
\[
u^\top M(p) u
=
\sum_{k=1}^K p_k u_k^2
-
\left(\sum_{k=1}^K p_k u_k\right)^2.
\]
This equals $\mathrm{Var}_p(u_D)$,
the variance of the scalar random variable $u_D$ under distribution $p$.

Thus,
\[
\|M(p)\|_{\mathrm{op}}
=
\sup_{\|u\|_2=1}
\mathrm{Var}_p(u_D).
\]

Popoviciu's variance inequality and $\|u\|_2=1$ give
\[
\mathrm{Var}_p(u_D)
\le \frac{1}{4}\left(\max_i u_i-\min_i u_i\right)^2
\le \frac{1}{2}\|u\|_2^2
=\frac{1}{2}.
\]
The second inequality follows because
$|u_i-u_j|\le \sqrt{2}\,(u_i^2+u_j^2)^{1/2}\le\sqrt{2}\|u\|_2$.
Equality is attained by
$p=(1/2,1/2,0,\dots,0)$ and
$u=(1/\sqrt{2},-1/\sqrt{2},0,\dots,0)^\top$. Therefore
\[
\|M(p)\|_{\mathrm{op}}\le\frac{1}{2}
\quad\text{for all } p\in\Delta_K.
\]
\end{proof}

\subsection{Local KL Expansion in Natural Parameters}

\begin{lemma}[Local KL expansion for categorical exponential family]
\label{lem:kl-expansion}
Let $p_\theta$ denote the categorical distribution
with reduced natural parameters $\theta\in\mathbb{R}^{K-1}$.
Then for sufficiently small $\delta$,
  \[
\mathrm{KL}(p_\theta \| p_{\theta+\delta})
=
\frac{1}{2} \delta^\top I_{\mathrm{red}}(\theta)\delta
+
o(\|\delta\|^2),
\] 
where
  \[
I_{\mathrm{red}}(\theta)
=
\mathrm{diag}(\bar p_\theta)-\bar p_\theta\bar p_\theta^\top,
\qquad
\bar p_\theta=(p_\theta(1),\ldots,p_\theta(K-1))^\top,
\] 
is the Fisher information matrix  in reduced natural coordinates.
\end{lemma}

\begin{proof}
The categorical model forms an exponential family with
log-partition function
\[
A(\theta)
=
\log\!\left(
1+\sum_{k=1}^{K-1} e^{\theta_k}
\right).
\]
Standard exponential-family theory gives
  \[
\nabla A(\theta)=\bar p_\theta,
\qquad
\nabla^2 A(\theta)=I_{\mathrm{red}}(\theta).
\] 

The KL divergence between two exponential-family members satisfies
\[
\mathrm{KL}(p_\theta\|p_{\theta+\delta})
=
A(\theta+\delta)-A(\theta)
-
\nabla A(\theta)^\top\delta.
\]
A second-order Taylor expansion of $A$ around $\theta$
yields the stated formula.
See \citet{amari2000methods} for a general treatment.
\end{proof}

\subsection{Proof of Proposition~\ref{prop:pop-min}}
\label{app:regression}

\begin{proof}
Let $v$ be any measurable function with $\mathbb E\|v(X)\|^2<\infty$ and define
$v^\star(X)=\mathbb E[U\mid X]$.
Write
\[
U - v(X) = (U - v^\star(X)) + (v^\star(X)-v(X)).
\]
Taking squared norms and expanding,
\[
\|U-v(X)\|^2
=
\|U-v^\star(X)\|^2 + \|v^\star(X)-v(X)\|^2
+2\langle U-v^\star(X),\, v^\star(X)-v(X)\rangle.
\]
Taking expectations, the cross term vanishes because
\[
\mathbb E\big[\langle U-v^\star(X),\, v^\star(X)-v(X)\rangle\big]
=
\mathbb E\Big[\,
\mathbb E\big[\langle U-v^\star(X),\, v^\star(X)-v(X)\rangle \mid X\big]
\Big],
\]
and conditional on $X$, $v^\star(X)-v(X)$ is deterministic while
$\mathbb E[U-v^\star(X)\mid X]=0$. Hence the cross term is $0$, and
\[
\mathbb E\|U-v(X)\|^2
=
\mathbb E\|U-v^\star(X)\|^2 + \mathbb E\|v^\star(X)-v(X)\|^2
\ge
\mathbb E\|U-v^\star(X)\|^2,
\]
showing $v^\star$ is a minimizer.
\end{proof}

\subsection{Proof of Lemma~\ref{lem:sieve}}
\label{app:sieve}
\begin{proof}
Recall the regression model
  \[
X=(Y_t,t), \qquad U=Y_1-Y_0,
\] 
and define
\[
v^\star(x)=\mathbb E[U\mid X=x].
\]

Let
\[
L(v)=\mathbb E\|v(X)-U\|^2,
\qquad
\hat L_n(v)
=
\frac{1}{n}\sum_{i=1}^n\|v(X_i)-U_i\|^2.
\]

Define the empirical minimizer
\[
\hat v_m
\in
\arg\min_{v\in\mathcal F_m}
\widehat L_n(v),
\]
and let
\[
v_m^\star
\in
\arg\min_{v\in\mathcal F_m}
L(v)
\]
denote the population projection.

Because $\mathcal F_m$ is a linear subspace of $L^2(P_X)$,
$v_m^\star$ is the $L^2(P_X)$-projection of $v^\star$ onto $\mathcal F_m$.
Hence the Pythagorean identity holds:
\[
\|v-v^\star\|_{L^2(P_X)}^2
=
\|v-v_m^\star\|_{L^2(P_X)}^2
+
\|v_m^\star-v^\star\|_{L^2(P_X)}^2,
\qquad \forall v\in\mathcal F_m.
\]

For each realized training sample, $\hat v_m\in\mathcal F_m$, so the same
identity applies to the data-dependent empirical minimizer:
\[
\|\hat v_m-v^\star\|_{L^2(P_X)}^2
=
\|\hat v_m-v_m^\star\|_{L^2(P_X)}^2
+
\|v_m^\star-v^\star\|_{L^2(P_X)}^2.
\]

Taking expectations yields
\begin{equation}
\label{eq:bias_variance_appendix}
\mathbb E\|\hat v_m-v^\star\|_{L^2(P_X)}^2
=
\mathbb E\|\hat v_m-v_m^\star\|_{L^2(P_X)}^2
+
\|v_m^\star-v^\star\|_{L^2(P_X)}^2.
\end{equation}

We now bound the two terms separately.

First, under Assumption~\ref{ass:velocity_regularity}(iii),
the conditional variance of $U$ is uniformly bounded.
Since $\mathcal F_m$ is an $m$-dimensional linear space
(Assumption~\ref{ass:velocity_regularity}(iv)),
Theorem~11.1 of \citet{Gyorfi2002}
implies that the least squares estimator satisfies
\[
\mathbb E\|\hat v_m-v_m^\star\|_{L^2(P_X)}^2
\le
C_1 \frac{m}{n},
\]
for a constant $C_1$ depending only on the variance bound.

Second, under Assumption~\ref{ass:velocity_regularity}(i),
$v^\star\in\mathcal H^\alpha(\Omega)$.
For standard spline or wavelet sieve spaces,
classical approximation theory yields
\[
\|v_m^\star-v^\star\|_{L^2(P_X)}^2
\le
C_2 m^{-2\alpha/d},
\]
where $d$ is the input dimension of $X$.

Substituting these bounds into
\eqref{eq:bias_variance_appendix}
gives
\[
\mathbb E\|\hat v_m-v^\star\|_{L^2(P_X)}^2
\le
C_1\frac{m}{n}
+
C_2 m^{-2\alpha/d}.
\]

Absorbing constants completes the proof.
\end{proof}

\subsection{Proof of Lemma~\ref{lem:bridge}}
\label{app:bridge}
\begin{proof}
By Assumption~\ref{ass:density}, for $g\ge 0$,
\[
\begin{aligned}
\mathbb E_{\mathrm{path}}[g(X)]
&=
\int g(x)\, d\mu_{\mathrm{path}}(x)\\
&=
\int g(x)\frac{d\mu_{\mathrm{path}}}{d\mu_{\mathrm{train}}}(x)\,
d\mu_{\mathrm{train}}(x)\\
&\le
\kappa \int g(x)\, d\mu_{\mathrm{train}}(x)
=
\kappa\, \mathbb E_{\mathrm{train}}[g(X)].
\end{aligned}
\]
Taking $g(x)=\|\hat v(x)-v^\star(x)\|^2$ yields the stated inequality.
\end{proof}

\subsection{Proof of Theorem~\ref{thm:mixed-rate}}
\label{app:mix-rate}
\begin{proof}
We treat the logit marginal component and the class-conditional
continuous component separately and combine the bounds according to
the decomposition of $\mathcal{D}_{\mathrm{mix}}$.

By Corollary~  \ref{cor:mixed_fidelity_fm}, the mixed discrepancy satisfies
  \[
\begin{aligned}
\mathcal{D}_{\mathrm{mix}}(P^\star,\widehat P)
\le\;&
\varepsilon
+
\frac{C_{\mathrm{sm}}}{2} e^{L_Z}\int_0^1
\Big(\mathbb{E}_{\mathrm{path}}
\| \hat v_Z(t,Z_t^\star)-v_Z^\star(t,Z_t^\star)\|^2\Big)^{1/2}\,dt \\
&+
\sum_{k=1}^K \pi_k^\star\, e^{L_k}
\int_0^1
\Big(\mathbb{E}_{\mathrm{path}}
\| \hat v_k(t,X_{t,k}^\star)-v_k^\star(t,X_{t,k}^\star)\|^2\Big)^{1/2}\,dt,
\end{aligned}
\] 
where $(Z_t^\star)$ and $(X_{t,k}^\star)$ denote the corresponding
population flows.

Since $\sqrt{\cdot}$ is concave, Jensen's inequality implies
\[
\int_0^1 \sqrt{g(t)}\,dt
\le
\Big(\int_0^1 g(t)\,dt\Big)^{1/2}
\]
for any nonnegative function $g$.
Applying this inequality to each integral yields
  \[
\begin{aligned}
\mathcal{D}_{\mathrm{mix}}(P^\star,\widehat P)
\le\;&
\varepsilon
+
\frac{C_{\mathrm{sm}}}{2} e^{L_Z}
\Bigg(
\int_0^1
\mathbb{E}_{\mathrm{path}}
\| \hat v_Z(t,Z_t^\star)-v_Z^\star(t,Z_t^\star)\|^2
\,dt
\Bigg)^{1/2} \\
&+
\sum_{k=1}^K \pi_k^\star\, e^{L_k}
\Bigg(
\int_0^1
\mathbb{E}_{\mathrm{path}}
\| \hat v_k(t,X_{t,k}^\star)-v_k^\star(t,X_{t,k}^\star)\|^2
\,dt
\Bigg)^{1/2}.
\end{aligned}
\] 

We now relate the path-integrated mean squared errors to the
training-distribution mean squared errors.
For each component, the regression input consists of $(t,Z_t)$
in the logit case and $(t,X_t)$ in the $k$-th continuous case.
By Lemma~\ref{lem:bridge} and Assumption~  \ref{ass:density},
there exists a constant $\kappa>0$ such that
\[
\int_0^1
\mathbb{E}_{\mathrm{path}}
\| \hat v(\cdot)-v^\star(\cdot)\|^2\,dt
\le
\kappa\,
\mathbb{E}_{\mathrm{train}}
\| \hat v(X)-v^\star(X)\|^2.
\]

Substituting this bound into the previous inequality and absorbing
constants ($\kappa$, $e^{L_Z}$, $e^{L_k}$, $C_{\mathrm{sm}}$) yields
  \[
\begin{aligned}
\mathcal{D}_{\mathrm{mix}}(P^\star,\widehat P)
\le\;&
\varepsilon
+
C_1
\Big(
\mathbb{E}_{\mathrm{train}}
\|\hat v_Z(X)-v_Z^\star(X)\|^2
\Big)^{1/2} \\
&+
\sum_{k=1}^K
\pi_k^\star\, C_{2,k}
\Big(
\mathbb{E}_{\mathrm{train}}
\|\hat v_k(X)-v_k^\star(X)\|^2
\Big)^{1/2}.
\end{aligned}
\] 

Applying Lemma~\ref{lem:sieve} to each regression problem gives the
corresponding mean squared error rates.
For the logit component, the input dimension is $d_Z+1$
with Hölder smoothness $\alpha_Z$, so that
\[
\mathbb{E}_{\mathrm{train}}
\|\hat v_Z-v_Z^\star\|^2
\lesssim
n^{-2\alpha_Z/(2\alpha_Z+d_Z+1)}.
\]
Taking square roots yields
\[
C_1
\Big(
\mathbb{E}_{\mathrm{train}}
\|\hat v_Z-v_Z^\star\|^2
\Big)^{1/2}
\lesssim
n^{-\alpha_Z/(2\alpha_Z+d_Z+1)}.
\]

For the $k$-th class-conditional component, the effective sample size
is $n\pi_k^\star$, the input dimension is $p+1$, and the smoothness index is
$\alpha_c$. Lemma~\ref{lem:sieve} therefore gives
\[
\mathbb{E}_{\mathrm{train}}
\|\hat v_k-v_k^\star\|^2
\lesssim
(n\pi_k^\star)^{-2\alpha_c/(2\alpha_c+p+1)}.
\]
Taking square roots and multiplying by $\pi_k^\star$ yields
\[
\pi_k^\star\, C_{2,k}
\Big(
\mathbb{E}_{\mathrm{train}}
\|\hat v_k-v_k^\star\|^2
\Big)^{1/2}
\lesssim
\pi_k^\star\,(n\pi_k^\star)^{-\alpha_c/(2\alpha_c+p+1)}.
\]

Combining the bounds for the logit and continuous components
establishes the claimed result.
\end{proof}

\newpage
\section{Simulation Data-Generating Details}
\label{app:simulation_dgp}

This section gives the full data-generating specification used in Section~7. The joint categorical state is
\[
G=(A,B,Y),\qquad
A\in\{0,1,2,3\},\quad
B\in\{0,1,2\},\quad
Y\in\{0,1\}.
\]
The rare cell is
\[
g_{\mathrm{rare}}=(A=3,B=2,Y=1),
\]
with probability
\[
\mathbb P(G=g_{\mathrm{rare}})=\rho,\qquad \rho\in\{0.05,0.01\}.
\]
For all other cells $g=(a,b,y)\ne g_{\mathrm{rare}}$, probabilities are assigned proportionally to
\[
q_{aby}
=0.63^a0.72^b0.58^y
\cdot 1.45^{\mathbf 1\{(a,b)\in\{(0,0),(1,1),(2,0)\}\}}
\cdot 0.70^{\mathbf 1\{y=1,b=2\}}.
\]
Thus,
\[
\pi_{aby}^\star=\mathbb P(G=(a,b,y))
=
\begin{cases}
\rho, & (a,b,y)=g_{\mathrm{rare}},\\[4pt]
(1-\rho)\,
\dfrac{q_{aby}}{\sum_{(a',b',y')\ne g_{\mathrm{rare}}}q_{a'b'y'}},
& (a,b,y)\ne g_{\mathrm{rare}}.
\end{cases}
\]

Conditional on $G=(a,b,y)$,
\[
X_{\mathrm{num}}\mid G=(a,b,y)\sim N(\mu_{aby},\Sigma_{aby}).
\]
The conditional mean is
\[
\mu_{aby}=\alpha_a+\beta_b+\gamma_y+\delta_{aby},
\]
where
\[
\begin{aligned}
\alpha_0&=(0,0,0,0),&
\alpha_1&=(0.55,-0.20,0.15,0),\\
\alpha_2&=(-0.40,0.45,0.20,-0.15),&
\alpha_3&=(0.20,0.10,-0.35,0.55),
\end{aligned}
\]
\[
\beta_0=(0,0,0,0),\qquad
\beta_1=(0.18,0.35,-0.10,0.25),\qquad
\beta_2=(-0.25,0.10,0.40,-0.20),
\]
and
\[
\gamma_0=(0,0,0,0),\qquad
\gamma_1=(0.35,-0.15,0.30,0.20).
\]
The interaction term is
\[
\delta_{aby}=
\begin{pmatrix}
0.10\mathbf 1\{a=3,y=1\}-0.07\mathbf 1\{b=2\}\\
0.12\mathbf 1\{a=1,b=1\}-0.05y\\
0.08\mathbf 1\{a=2,y=1\}+0.06\mathbf 1\{b=2,y=0\}\\
0.10\mathbf 1\{a=3,b=2\}-0.04a
\end{pmatrix}.
\]
The covariance matrix is
\[
\Sigma_{aby}=D_{aby}R_{aby}D_{aby}+10^{-6}I_4,
\]
where $D_{aby}=\mathrm{diag}(s_{aby})$ and
\[
s_{aby}=
\begin{pmatrix}
0.85+0.07a\\
0.95+0.06b\\
0.90+0.05y+0.03((a+b)\bmod 2)\\
0.80+0.04a+0.04y
\end{pmatrix}.
\]
The correlation matrix $R_{aby}$ has diagonal entries equal to one and off-diagonal entries
\[
\begin{aligned}
(R_{aby})_{12}=(R_{aby})_{21}&=0.20-0.04b,\\
(R_{aby})_{13}=(R_{aby})_{31}&=-0.10+0.03a,\\
(R_{aby})_{24}=(R_{aby})_{42}&=0.15-0.05y,\\
(R_{aby})_{34}=(R_{aby})_{43}&=0.10+0.02b,
\end{aligned}
\]
with all other off-diagonal entries equal to zero.

 \section{Additional Real-Data Results}

\subsection{Distributional Diagnostics}
\label{app:real_data_plots}

In addition to the quantitative results in Section~\ref{sec:real_data}, we provide qualitative distributional diagnostics based on the original benchmark partition of each dataset. These plots are not aggregated over the nine additional random partitions and should be interpreted as representative illustrations. The main evidence is the mean and standard deviation over the ten data splits reported in the main text.

Figures~\ref{fig:num_adult_v4}--\ref{fig:num_buddy_v4} compare representative marginal numerical distributions, Figures~\ref{fig:cat_adult_v4}--\ref{fig:cat_buddy_v4} compare marginal categorical distributions, and Figures~\ref{fig:cond_adult_v4}--\ref{fig:cond_buddy_v4} compare conditional numerical distributions of the form
\[
P(X_{\mathrm{num}}\mid X_{\mathrm{cat}},y).
\]
The conditional plots are particularly relevant to the block-conditional factorization
\[
P(y)P(X_{\mathrm{cat}}\mid y)
P(X_{\mathrm{num}}\mid X_{\mathrm{cat}},y).
\]

\renewcommand{\floatpagefraction}{0.72}
\setcounter{totalnumber}{4}

\begin{figure}[p]
\centering
\includegraphics[width=.92\textwidth,height=.82\textheight,keepaspectratio]{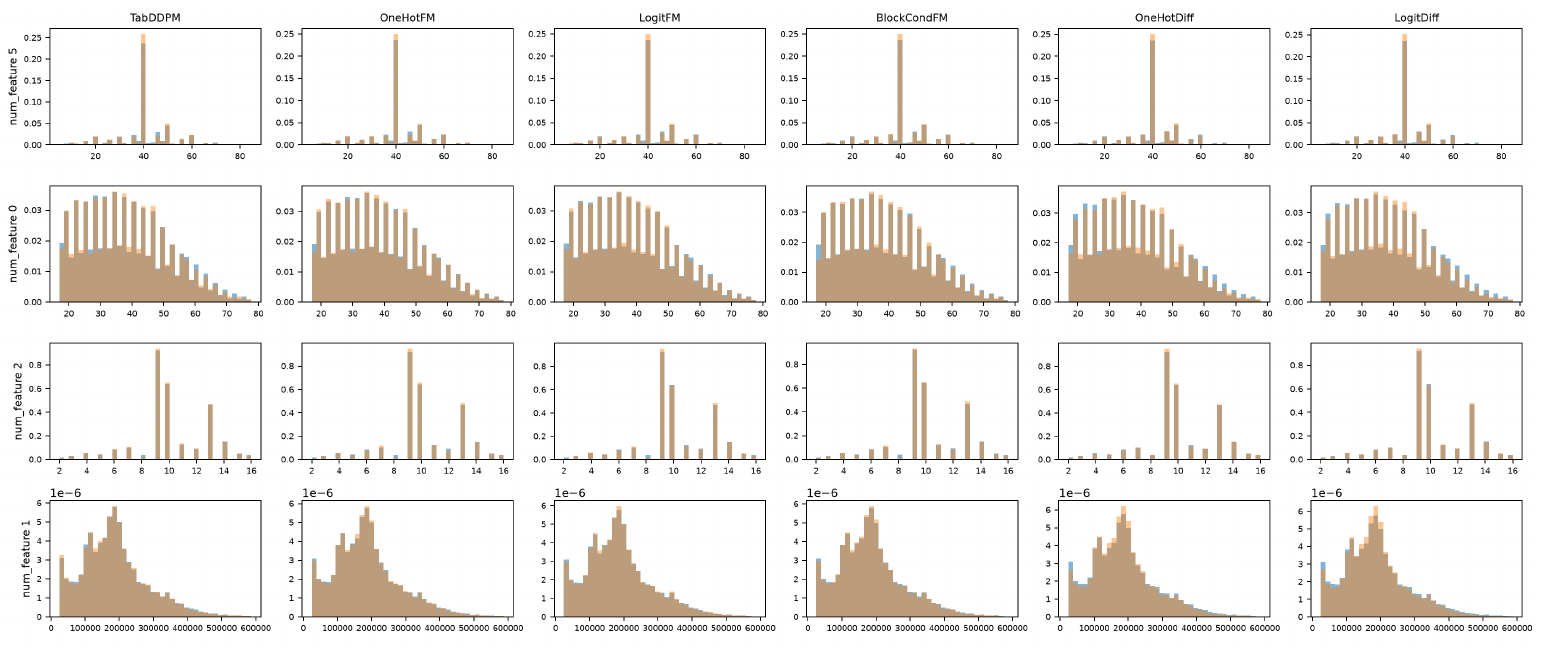}
\caption{Adult marginal numerical distributions on the original benchmark split. Columns show TabDDPM; one-hot, logit, and block-conditional logit Flow Matching; and one-hot/logit diffusion. Rows are selected challenging features. Blue/orange histograms represent real/synthetic observations; the axes give feature value and density. Greater overlap indicates better marginal fidelity. Broad shapes align, with local bin differences. This is a single-split diagnostic of univariate marginals.}
\label{fig:num_adult_v4}
\end{figure}

\begin{figure}[p]
\centering
\includegraphics[width=.92\textwidth,height=.82\textheight,keepaspectratio]{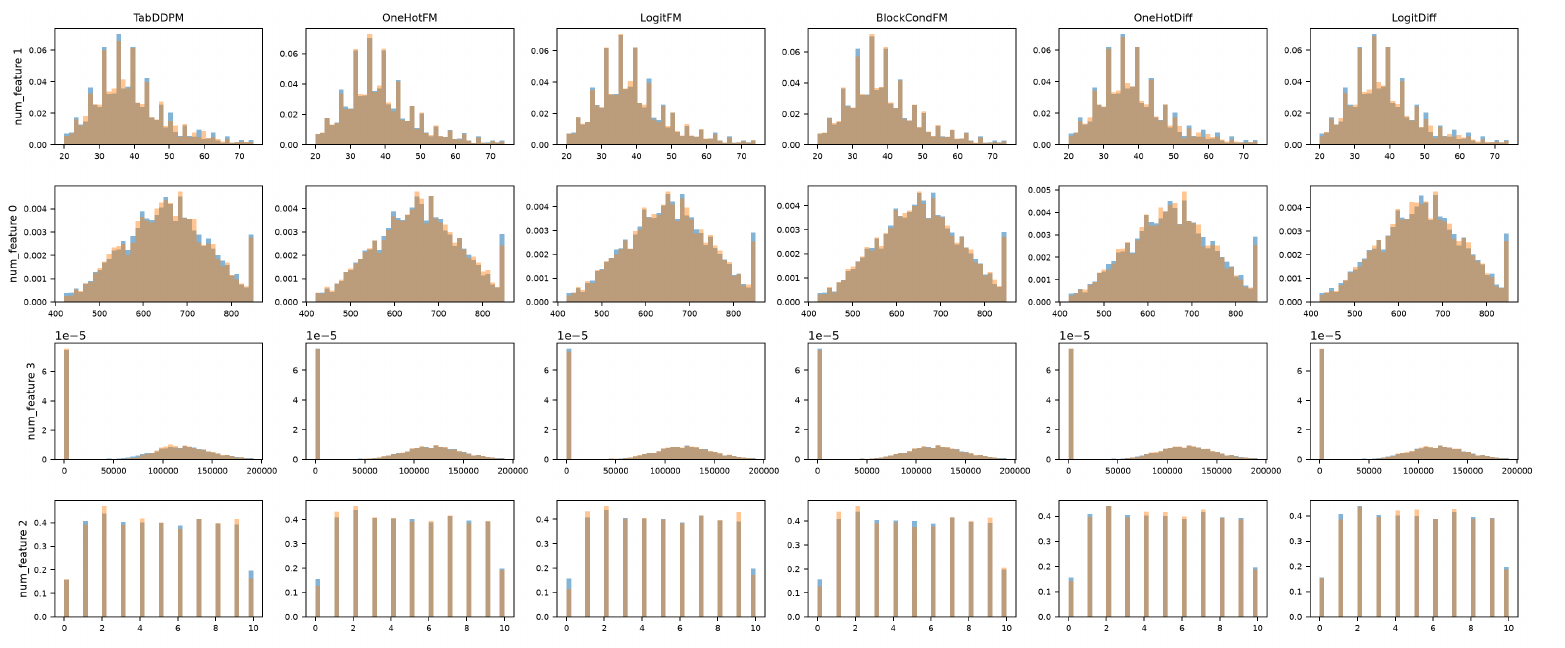}
\caption{Churn2 marginal numerical distributions on the original benchmark split. Columns show TabDDPM; one-hot, logit, and block-conditional logit Flow Matching; and one-hot/logit diffusion. Rows are selected challenging features. Blue/orange histograms represent real/synthetic observations; the axes give feature value and density. Greater overlap indicates better marginal fidelity. All methods recover the main distributional shapes, with modest local and endpoint discrepancies. This is a single-split diagnostic of univariate marginals.}
\label{fig:num_churn_v4}
\end{figure}

\begin{figure}[p]
\centering
\includegraphics[width=.92\textwidth,height=.82\textheight,keepaspectratio]{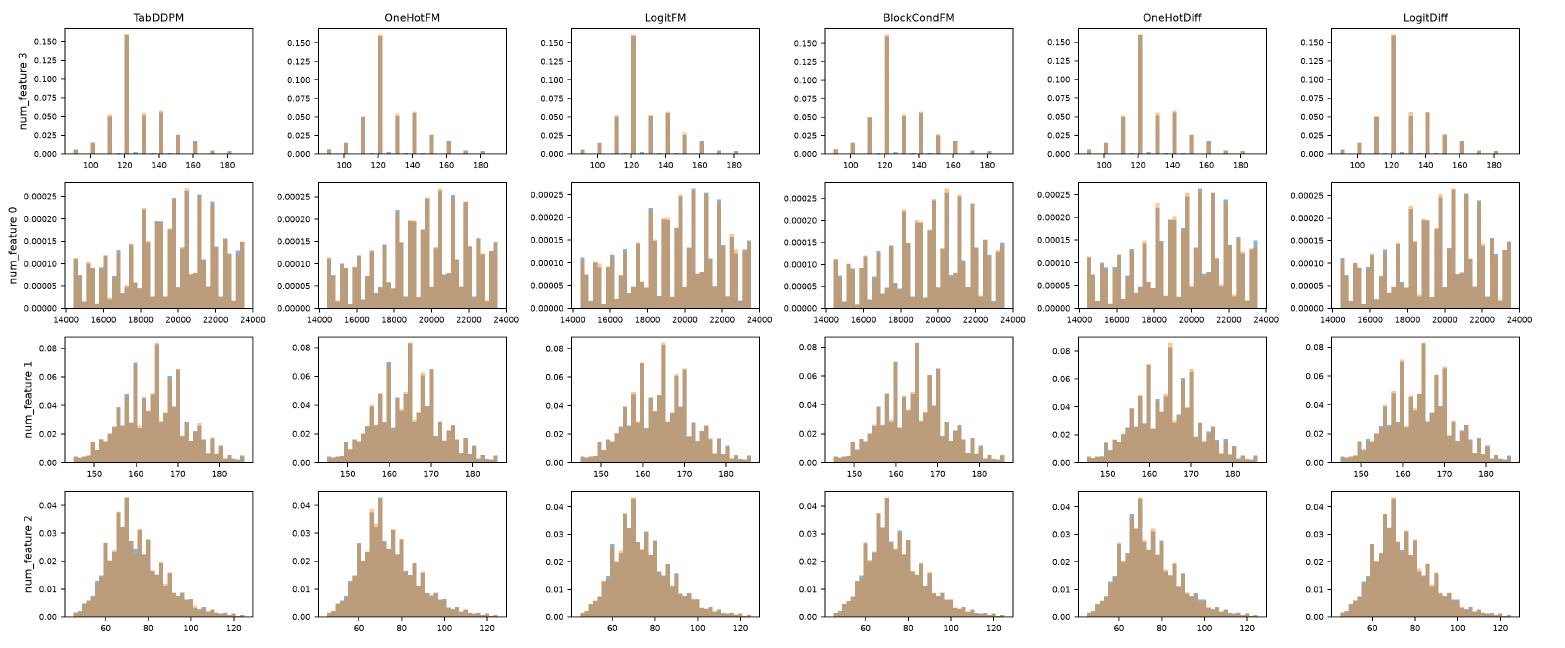}
\caption{Cardio marginal numerical distributions on the original benchmark split. Columns show TabDDPM; one-hot, logit, and block-conditional logit Flow Matching; and one-hot/logit diffusion. Rows are selected challenging features. Blue/orange histograms represent real/synthetic observations; the axes give feature value and density. Greater overlap indicates better marginal fidelity. Spikes, skewness, and broad shapes align closely, with local bin differences. This is a single-split diagnostic of univariate marginals.}
\label{fig:num_cardio_v4}
\end{figure}

\begin{figure}[p]
\centering
\includegraphics[width=.92\textwidth,height=.82\textheight,keepaspectratio]{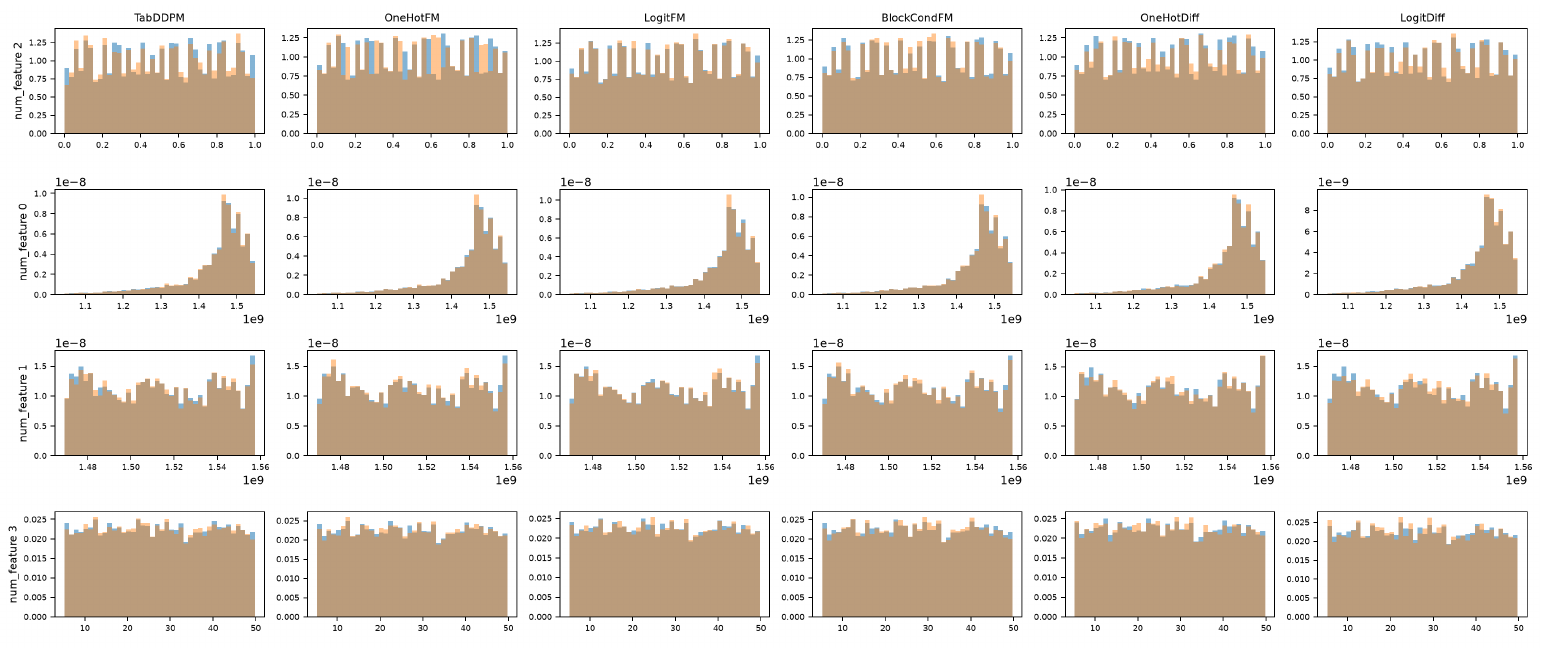}
\caption{Buddy marginal numerical distributions on the original benchmark split. Columns show TabDDPM; one-hot, logit, and block-conditional logit Flow Matching; and one-hot/logit diffusion. Rows are selected challenging features. Blue/orange histograms represent real/synthetic observations; the axes give feature value and density. Greater overlap indicates better marginal fidelity. All methods preserve the broad bounded and skewed shapes, with local frequency differences. This is a single-split diagnostic of univariate marginals.}
\label{fig:num_buddy_v4}
\end{figure}

\begin{figure}[p]
\centering
\includegraphics[width=.74\textwidth,height=.82\textheight,keepaspectratio]{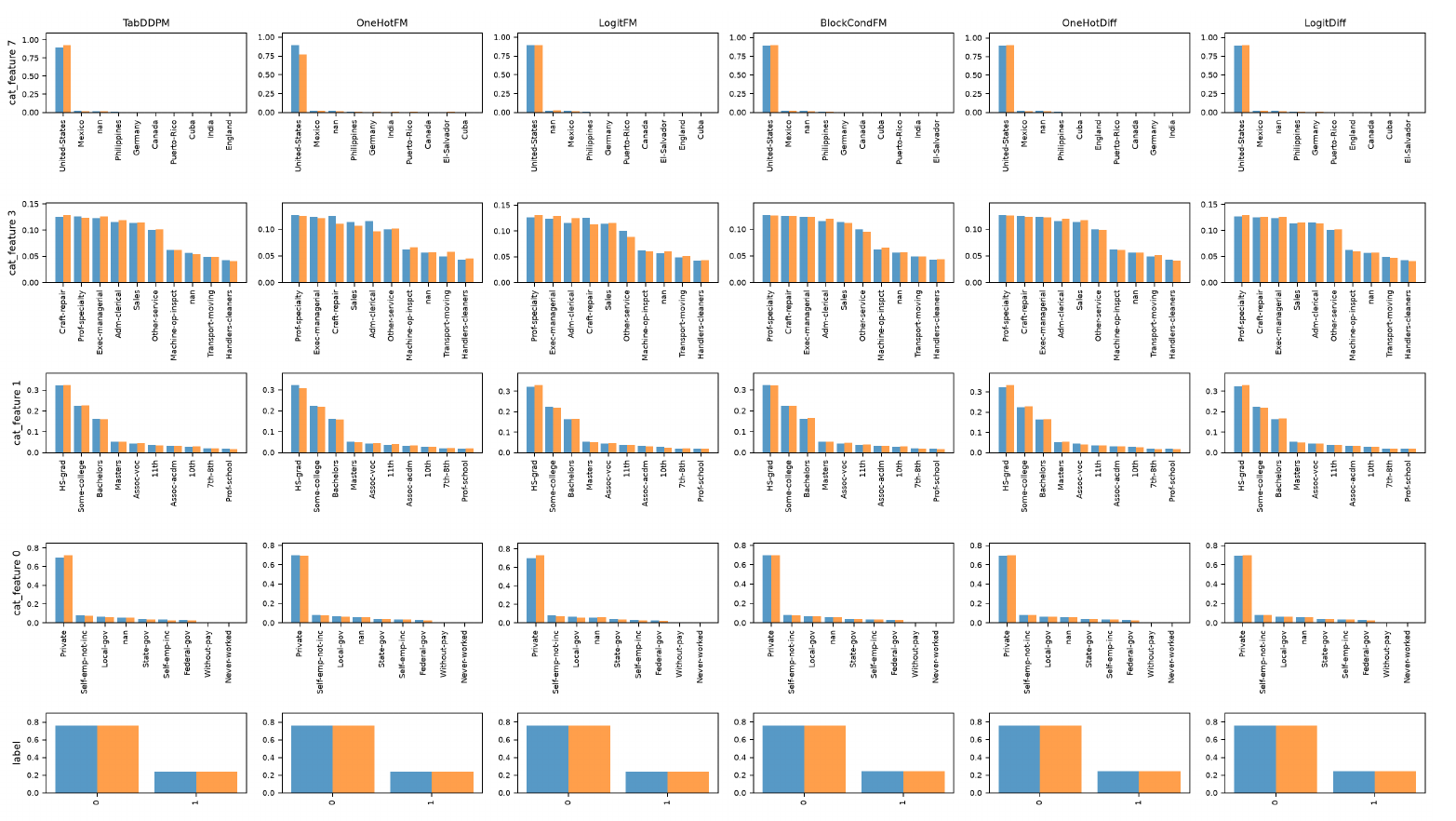}
\caption{Adult marginal categorical distributions on the original benchmark split. Columns show TabDDPM; one-hot, logit, and block-conditional logit Flow Matching; and one-hot/logit diffusion. Rows are selected predictors and the label. Blue/orange bars are real/synthetic category probabilities; smaller gaps indicate lower marginal TV. Most probabilities align, although OneHotFM underrepresents the dominant United-States level in feature 7. This is a single-split marginal diagnostic.}
\label{fig:cat_adult_v4}
\end{figure}

\begin{figure}[p]
\centering
\includegraphics[width=.74\textwidth,height=.82\textheight,keepaspectratio]{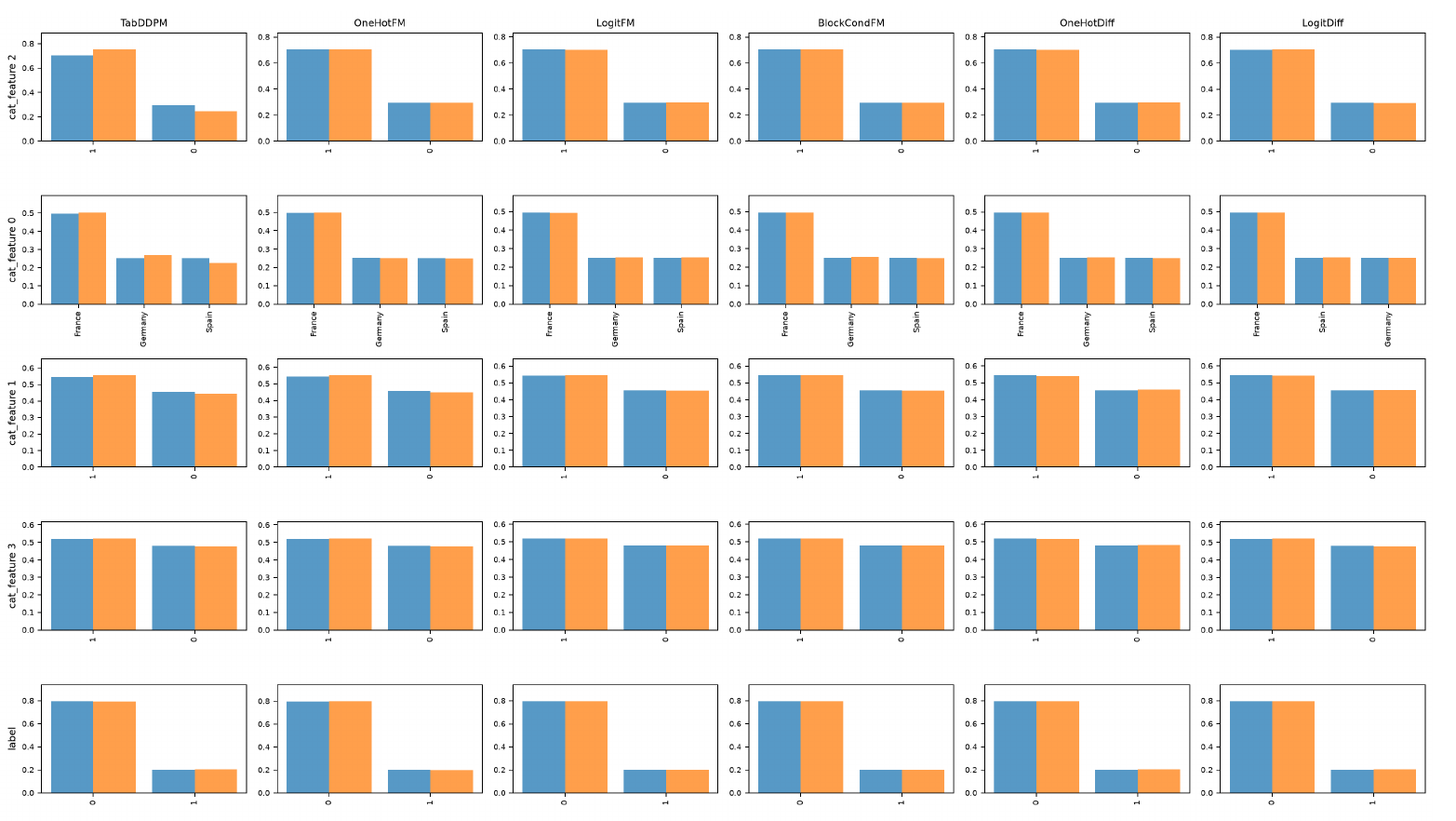}
\caption{Churn2 marginal categorical distributions on the original benchmark split. Columns show TabDDPM; one-hot, logit, and block-conditional logit Flow Matching; and one-hot/logit diffusion. Rows are selected predictors and the label. Blue/orange bars are real/synthetic category probabilities; smaller gaps indicate lower marginal TV. Most paired bars nearly coincide, with the largest displayed shift for TabDDPM in feature 2. This is a single-split marginal diagnostic.}
\label{fig:cat_churn_v4}
\end{figure}

\begin{figure}[p]
\centering
\includegraphics[width=.74\textwidth,height=.82\textheight,keepaspectratio]{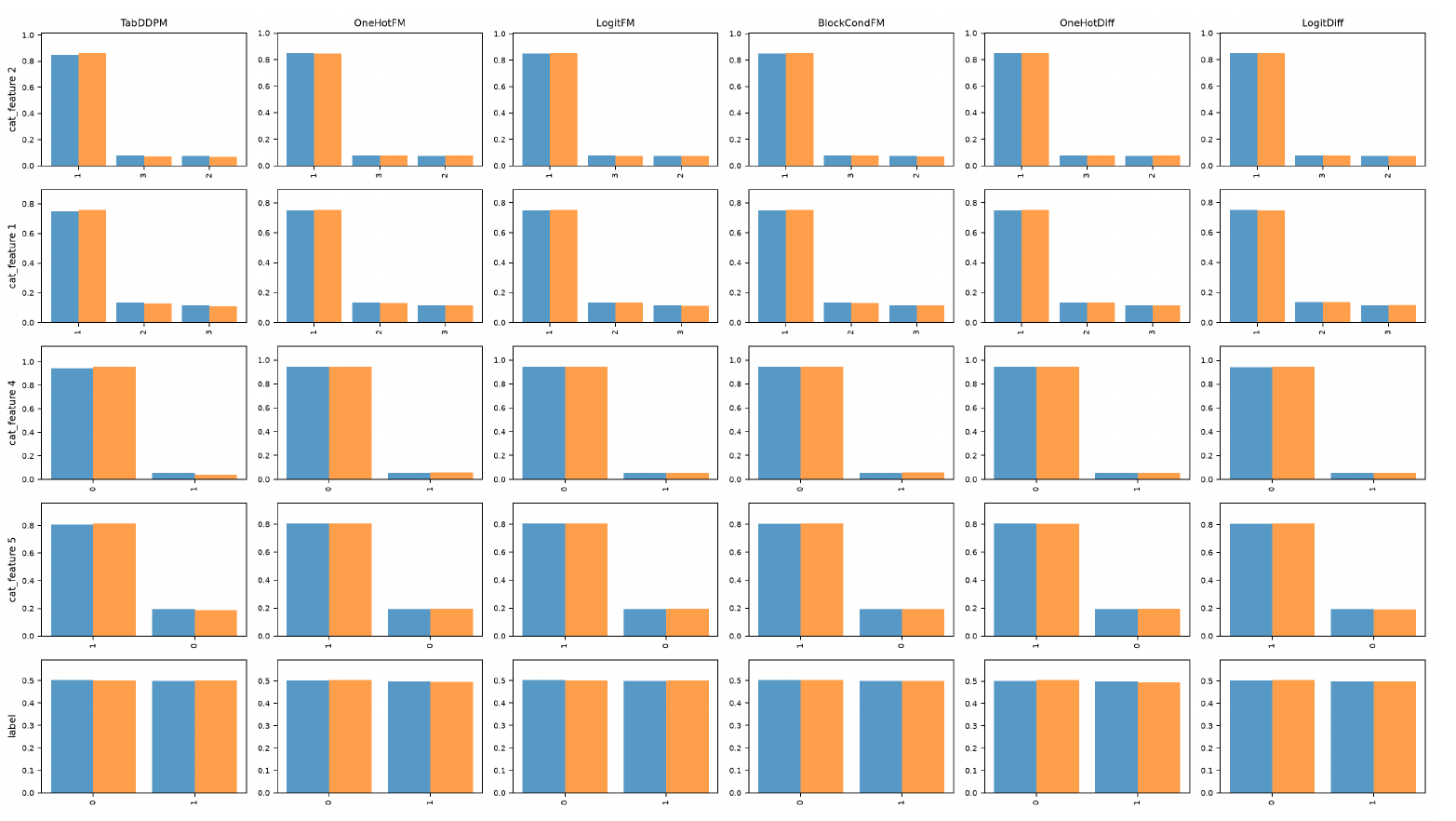}
\caption{Cardio marginal categorical distributions on the original benchmark split. Columns show TabDDPM; one-hot, logit, and block-conditional logit Flow Matching; and one-hot/logit diffusion. Rows are selected predictors and the label. Blue/orange bars are real/synthetic category probabilities; smaller gaps indicate lower marginal TV. The displayed proportions nearly coincide across variables and methods. This is a single-split marginal diagnostic.}
\label{fig:cat_cardio_v4}
\end{figure}

\begin{figure}[p]
\centering
\includegraphics[width=.74\textwidth,height=.82\textheight,keepaspectratio]{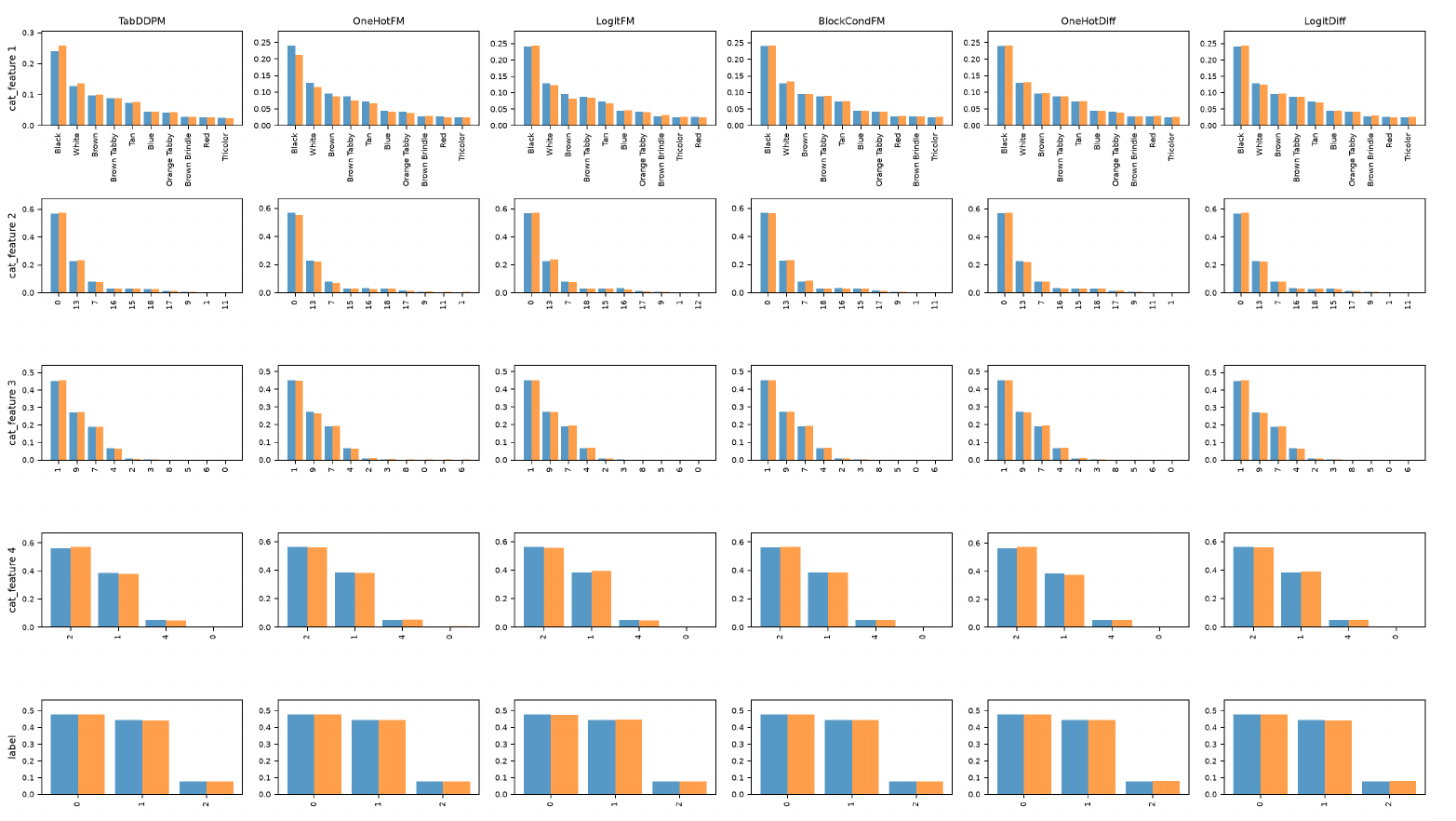}
\caption{Buddy marginal categorical distributions on the original benchmark split. Columns show TabDDPM; one-hot, logit, and block-conditional logit Flow Matching; and one-hot/logit diffusion. Rows are selected predictors and the three-level label. Blue/orange bars are real/synthetic category probabilities; smaller gaps indicate lower marginal TV. Binary and label proportions align closely; larger gaps occur in higher-cardinality predictors. This is a single-split marginal diagnostic.}
\label{fig:cat_buddy_v4}
\end{figure}

\begin{figure}[p]
\centering
\includegraphics[width=.82\textwidth,height=.82\textheight,keepaspectratio]{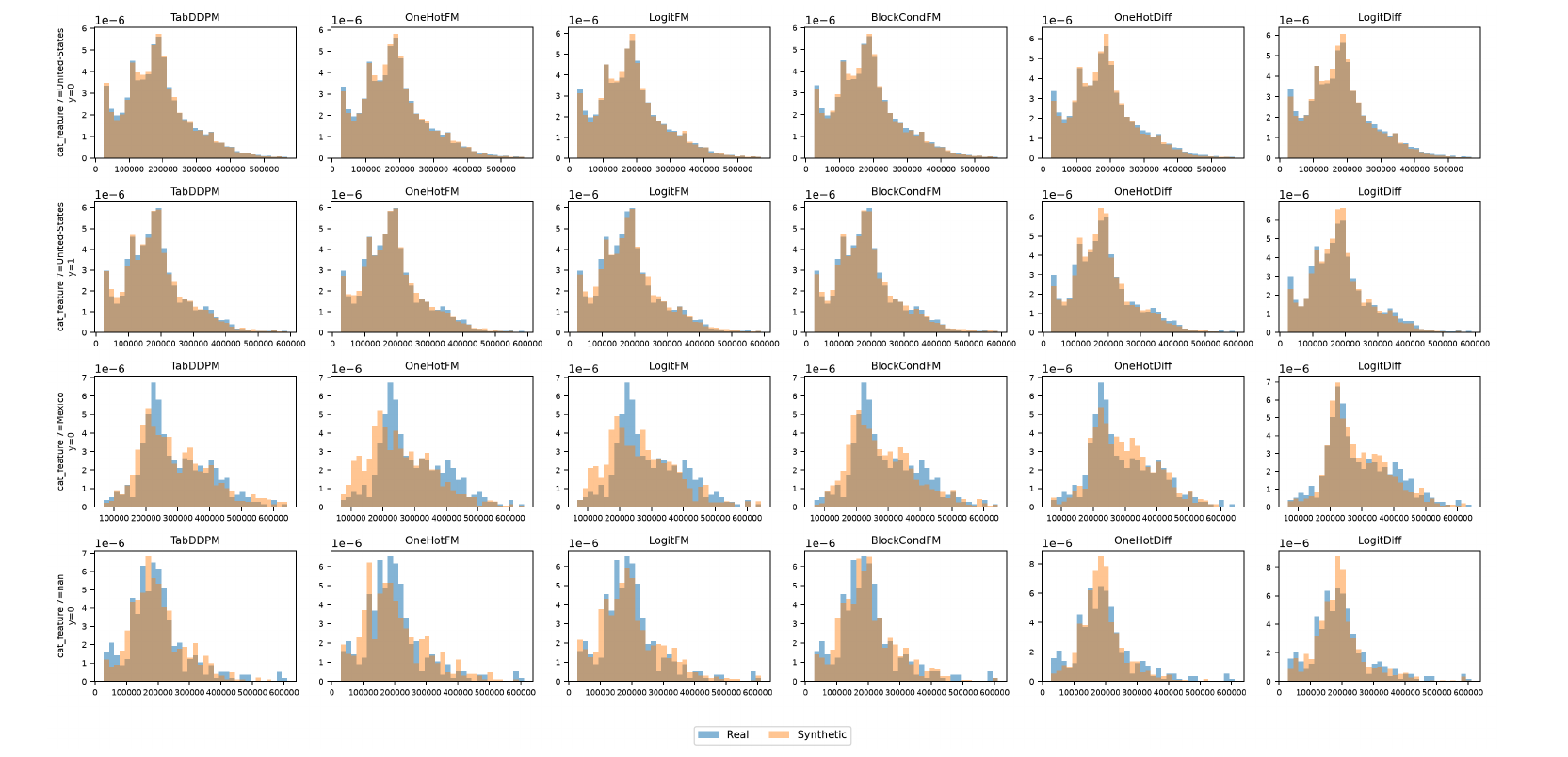}
\caption{Adult conditional numerical distributions on the original benchmark split. Columns show TabDDPM; one-hot, logit, and block-conditional logit Flow Matching; and one-hot/logit diffusion. Rows fix the categorical level and label shown at left and plot $P(X_{\mathrm{num}}\mid X_{\mathrm{cat}},y)$. Blue/orange histograms are real/synthetic; the axes give numerical value and density. Greater overlap indicates better conditional fidelity. The United-States cells align more closely than the displayed Mexico and Iran cells. This is a single-split diagnostic of selected high-support cells.}
\label{fig:cond_adult_v4}
\end{figure}

\begin{figure}[p]
\centering
\includegraphics[width=.82\textwidth,height=.82\textheight,keepaspectratio]{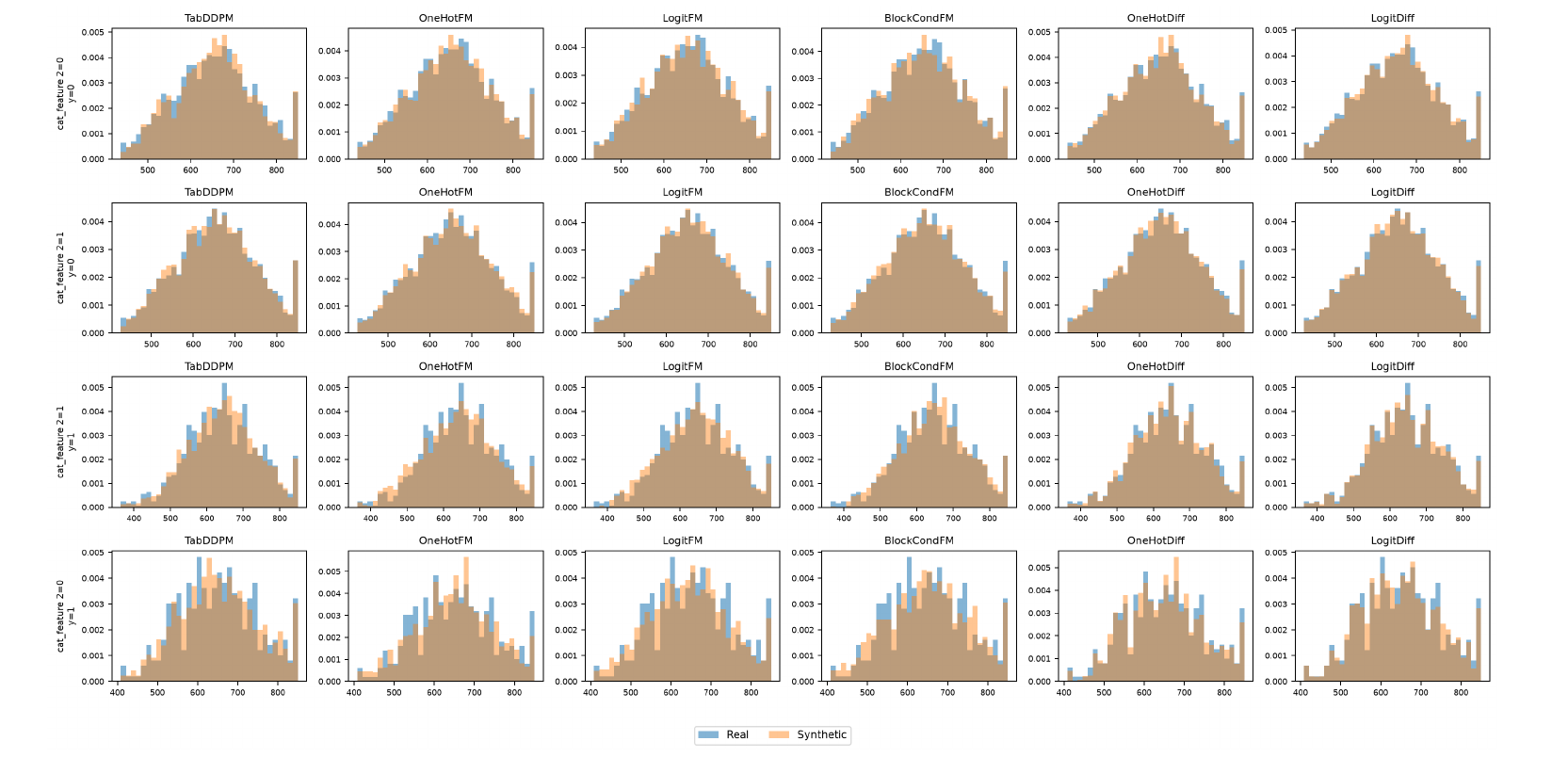}
\caption{Churn2 conditional numerical distributions on the original benchmark split. Columns show TabDDPM; one-hot, logit, and block-conditional logit Flow Matching; and one-hot/logit diffusion. Rows fix $(z,y)$ as shown at left and plot $P(X_{\mathrm{num}}\mid z,y)$. Blue/orange histograms are real/synthetic; the axes give numerical value and density. Greater overlap indicates better conditional fidelity. Broad shapes are recovered, with more binwise variation in the $y=1$ cells. This is a single-split diagnostic of selected high-support cells.}
\label{fig:cond_churn_v4}
\end{figure}

\begin{figure}[p]
\centering
\includegraphics[width=.82\textwidth,height=.82\textheight,keepaspectratio]{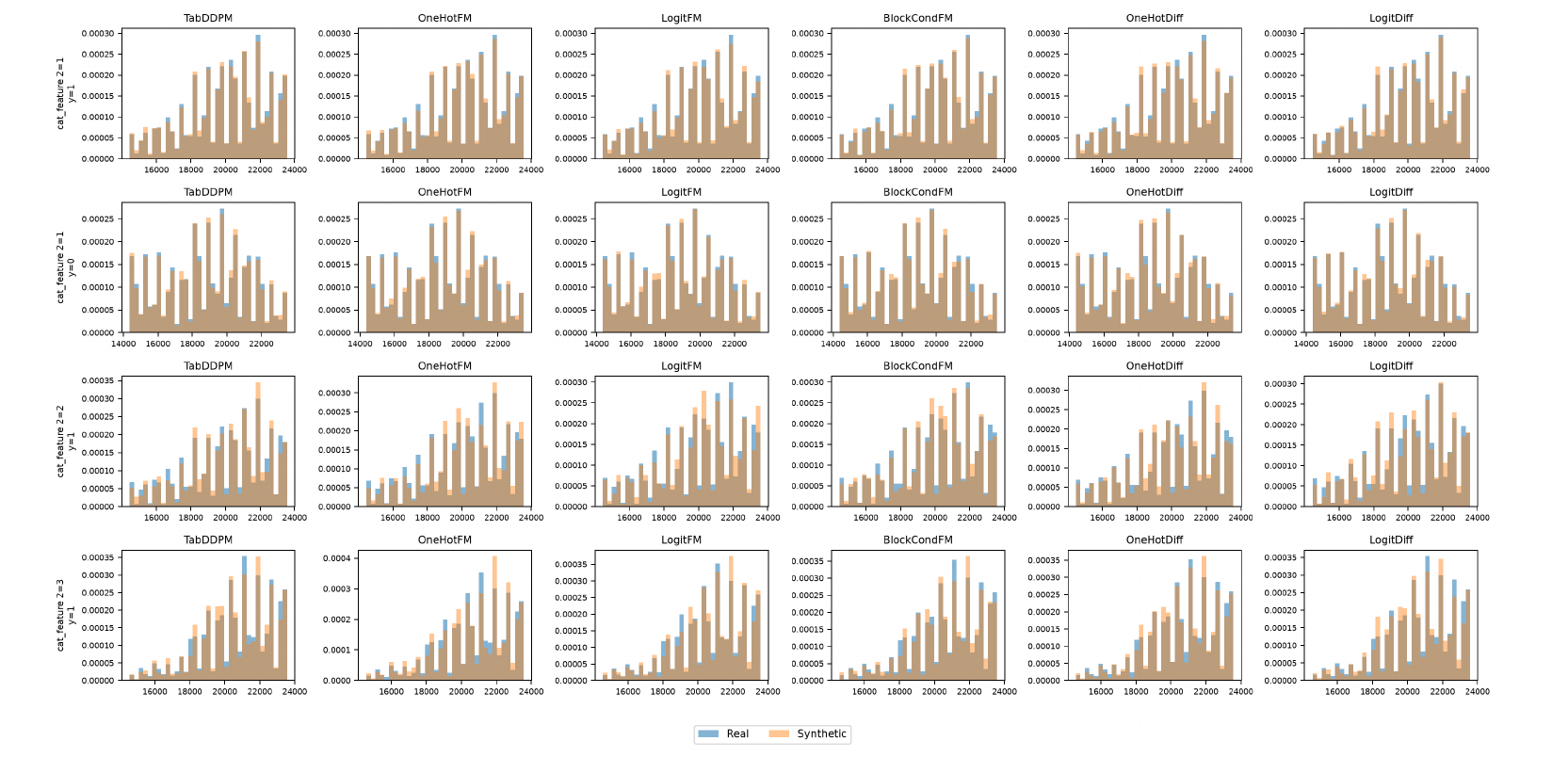}
\caption{Cardio conditional numerical distributions on the original benchmark split. Columns show TabDDPM; one-hot, logit, and block-conditional logit Flow Matching; and one-hot/logit diffusion. Rows fix $(z,y)$ as shown at left and plot $P(X_{\mathrm{num}}\mid z,y)$. Blue/orange histograms are real/synthetic; the axes give numerical value and density. Greater overlap indicates better conditional fidelity. The increasing multimodal pattern is recovered, with local peak and tail differences in the $y=1$ cells. This is a single-split diagnostic of selected high-support cells.}
\label{fig:cond_cardio_v4}
\end{figure}

\begin{figure}[p]
\centering
\includegraphics[width=.82\textwidth,height=.82\textheight,keepaspectratio]{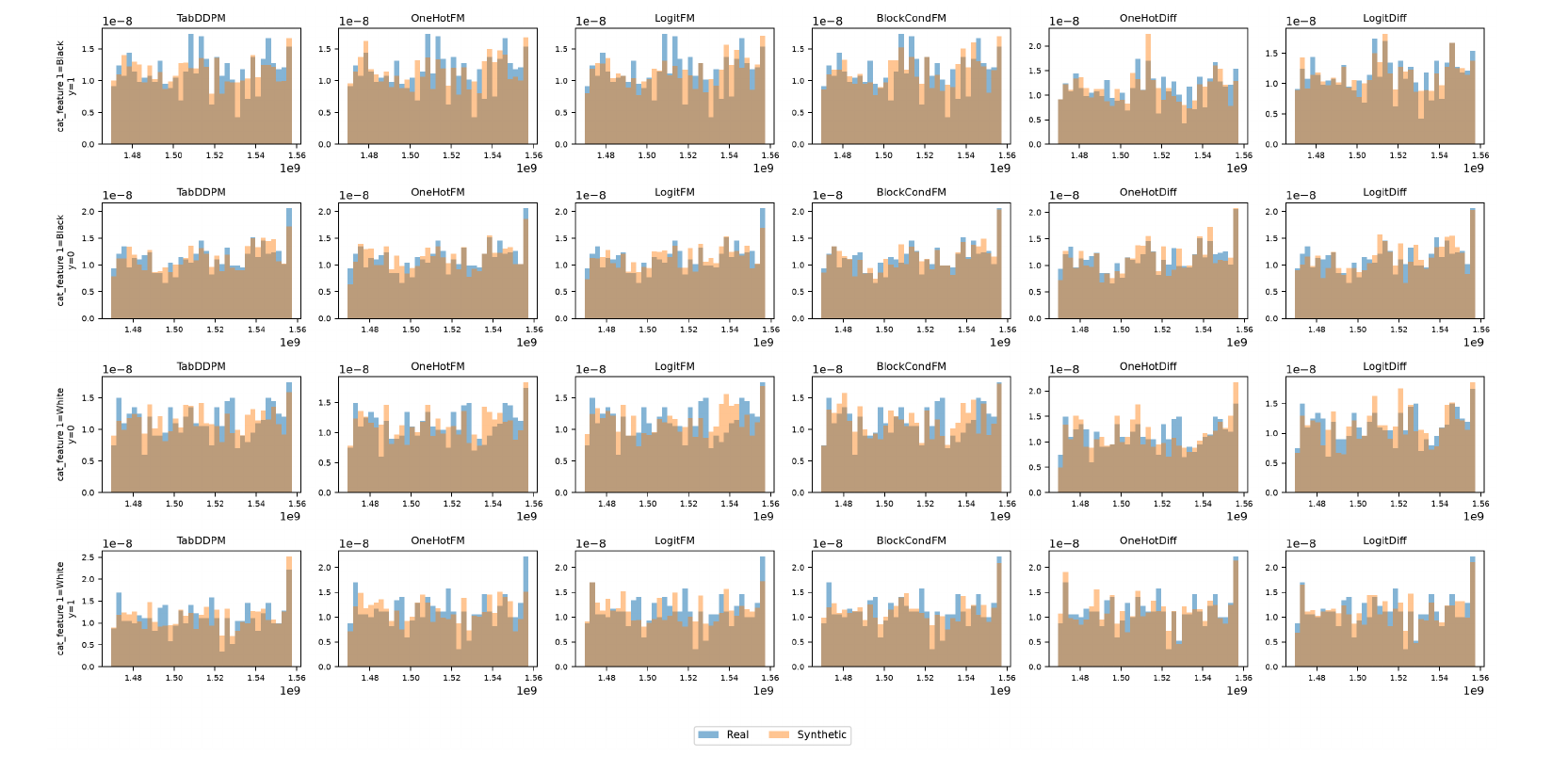}
\caption{Buddy conditional numerical distributions on the original benchmark split. Columns show TabDDPM; one-hot, logit, and block-conditional logit Flow Matching; and one-hot/logit diffusion. Rows fix race and label as shown at left and plot $P(X_{\mathrm{num}}\mid X_{\mathrm{cat}},y)$. Blue/orange histograms are real/synthetic; the axes give numerical value and density. Greater overlap indicates better conditional fidelity. Broad levels are reproduced, with the largest local gaps in the displayed White, $y=1$ cell. This is a single-split diagnostic of selected high-support cells.}
\label{fig:cond_buddy_v4}
\end{figure}

Overall, the marginal plots show that individual numerical and categorical feature distributions are often easier to match than the full mixed joint distribution. The conditional plots are more directly related to the proposed structured modeling approach: they visualize whether generated data preserve conditional distributions of the form $P(X_{\mathrm{num}}\mid X_{\mathrm{cat}},y)$. These diagnostics support the use of the conditional Wasserstein term and standardized $\mathcal{D}_{\mathrm{mix}}$ as the main quantitative   criteria  for real-data evaluation.

\end{document}